\documentclass{article}

\PassOptionsToPackage{numbers, compress}{natbib}

\usepackage[final]{neurips_2021}

\usepackage[utf8]{inputenc} \usepackage[T1]{fontenc}    \usepackage{hyperref}       \usepackage{url}            \usepackage{booktabs}       \usepackage{amsfonts}       \usepackage{nicefrac}       \usepackage{microtype}      \usepackage[table]{xcolor}  
\usepackage[varqu,varl,var0,scaled=0.97]{inconsolata}

\usepackage{amsmath,amssymb,amsfonts,amsthm,mathrsfs}
\usepackage{algpseudocode}
\usepackage{hyperref}
\usepackage{graphicx}
\usepackage{thm-restate}
\usepackage[normalem]{ulem}
\usepackage{xspace}
\usepackage{arydshln}
\usepackage{wrapfig}
\usepackage{enumitem}
\usepackage{framed}
\usepackage{booktabs}
\usepackage{tabularx}
\usepackage{multirow}
\usepackage{balance}
\usepackage{dsfont}\def\eps{{\epsilon}}
\usepackage{overpic}
\usepackage{algorithm}
\usepackage{overpic}
\usepackage{wrapfig}
\usepackage{cleveref}
\setlist{nosep}

\newcommand{\hy}{\hat{y}}

\newcommand{\ty}{\tilde{y}}
\newcommand{\bp}{\mathbf{p}}

\newcommand{\nc}{\newcommand}
\nc{\DMO}{\DeclareMathOperator}
\nc{\st}{\star}
\nc{\grad}{\nabla}
\nc\m[2]{m_{#1}(#2)}
\nc{\LabelSGD}{\texttt{LP-SGD}\xspace}
\nc{\LabelRRSGD}{\texttt{LP-RR-SGD}\xspace}
\nc{\LabelNormalSGD}{\texttt{LP-Normal-SGD\xspace}}
\nc{\MultiStage}{\texttt{Multi-Stage Training}\xspace}
\nc{\multiST}{\texttt{LP-MST}\xspace}
\nc{\oneST}{\texttt{LP-1ST}\xspace}
\nc{\twoST}{\texttt{LP-2ST}\xspace}
\nc{\threeST}{\texttt{LP-3ST}\xspace}
\nc{\fourST}{\texttt{LP-4ST}\xspace}
\nc{\RR}{\texttt{RandomizedResponse}\xspace}
\nc{\AShort}{\texttt{A}\xspace}
\nc{\BShort}{\texttt{B}\xspace}
\nc{\RShort}{\texttt{R}\xspace}
\nc{\dpsgd}{\texttt{DP-SGD}\xspace}
\nc{\RRShort}{\texttt{RR}\xspace}
\nc{\RRP}{\texttt{RRWithPrior}\xspace}
\nc{\RRtopk}{\texttt{RRTop-$k$}\xspace}
\nc{\RAPPOR}{\texttt{RAPPOR}\xspace}
\nc{\GenericLearner}{\texttt{GenericLearner}\xspace}
\nc{\BR}{\mathbb{R}}
\nc{\BM}{\mathbb{M}}
\nc{\BT}{\mathbb{T}}
\DMO{\Span}{span}
\nc{\BN}{\mathbb{N}}
\nc{\BZ}{\mathbb{Z}}
\nc{\ep}{\varepsilon}
\DMO{\height}{ht}
\renewcommand{\epsilon}{\varepsilon}
\nc{\ra}{\rightarrow}
\DMO{\Err}{err}
\DMO{\good}{good}
\DMO{\negpt}{neg-pt}
\DMO{\VV}{{V}}
\DMO{\LL}{{L}}
\nc{\err}[2]{\Err_{#1}(#2)}
\DeclareMathOperator*{\argmax}{arg\,max}
\DeclareMathOperator*{\argmin}{arg\,min}

\nc{\Rprot}{R}
\nc{\Sprot}{S}
\nc{\Aprot}{A}
\nc{\Pprot}{P}
\nc{\Pdist}{D}
\nc{\Qdist}{F}
\nc{\pdist}{d}
\nc{\qdist}{f}

\nc{\DP}{DP\xspace}

\nc{\SD}{\mathscr{D}}
\nc{\la}{\lambda}
\DMO{\KL}{KL}
\DMO{\Unif}{Unif}
\nc{\nn}{\varnothing}
\DMO{\SOA}{SOA}
\nc{\soa}[2]{\SOA_{#1}(#2)}
\nc{\soaf}[1]{\SOA_{#1}}
\DMO{\Red}{red}
\DMO{\Irred}{irred}
\nc{\Ired}{I^{\Red}}
\nc{\Iirred}{I^{\Irred}}

\DMO{\ssmp}{ssmp}
\DMO{\agg}{agg}
\DMO{\final}{final}

\nc{\pp}{p}
\nc{\PP}{P}
\nc{\QQ}{Q}
\nc{\DD}{D}
\nc{\lng}{\langle}
\nc{\rng}{\rangle}

\DMO{\obj}{Obj}
\nc{\RAP}{\RAPPOR}
\nc{\MR}{\mathcal{R}}
\nc{\MD}{\mathcal{D}}
\nc{\ML}{\mathcal{L}}
\nc{\di}{P}
\nc{\MO}{\mathcal{O}}
\nc{\MM}{\mathcal{M}}
\nc{\MZ}{\mathcal{Z}}
\nc{\MU}{\mathcal{U}}
\nc{\MP}{\mathcal{P}}
\nc{\poly}{\mathrm{poly}}
\DMO{\treesum}{TreeSum}
\DMO{\lapsum}{LapSum}
\DMO{\checksum}{CheckSum}
\nc{\MDts}{\MD_{\treesum}}
\nc{\MDls}{\MD_{\lapsum}}
\nc{\MDcs}{\MD_{\checksum}}
\nc{\MC}{\mathcal{C}}
\nc{\MT}{\mathcal{T}}
\nc{\MS}{\mathcal{S}}
\nc{\MX}{\mathcal{X}}
\nc{\MY}{\mathcal{Y}}
\nc{\MA}{\mathcal{A}}
\nc{\MB}{\mathcal{B}}
\nc{\MJ}{\mathcal{J}}
\nc{\MF}{\mathcal{F}}
\nc{\MG}{\mathcal{G}}
\nc{\MQ}{\mathcal{Q}}
\nc{\p}{\Pr}
\nc{\E}{\mathbb{E}}
\nc{\tablesize}{s}
\DMO{\Hist}{hist}
\nc{\hist}{\mathrm{hist}}

\nc{\ba}{\mathbf{a}}
\nc{\bx}{\mathbf{x}}
\nc{\bs}{\mathbf{s}}
\nc{\bv}{\mathbf{v}}
\nc{\bw}{\mathbf{w}}
\nc{\by}{\mathbf{y}}
\nc{\bz}{\mathbf{z}}

\DMO{\sr}{sr}
\DMO{\Med}{Med}
\DMO{\Ber}{Ber}
\DMO{\Bin}{Bin}
\DMO{\Had}{Had}

\nc{\ME}{\mathcal{E}}
\DMO{\View}{View}
\nc{\B}{B}
\nc{\M}{M}
\nc{\ha}{\kappa}
\DMO{\pre}{pre}
\nc{\MH}{\mathcal{H}}
\DMO{\Ldim}{Ldim}
\DMO{\Tdim}{Tdim}
\DMO{\sfat}{sfat}
\DMO{\fat}{fat}
\DMO{\vc}{VCdim}
\DMO{\FO}{FO}
\DMO{\CM}{CM}
\DMO{\scr}{scr}
\DMO{\hb}{\beta}
\nc{\MW}{\mathcal{W}}
\nc{\MN}{\mathcal{N}}

\nc{\BB}{\{0,1\}}
\nc{\bW}{\mathbf{W}}
\nc{\eell}{\ell}
\nc{\EELL}{L}
\nc{\q}{q}
\DMO{\ERM}{ERM}

\newtheorem{theorem}{Theorem}
\newtheorem{corollary}[theorem]{Corollary}
\newtheorem{proposition}[theorem]{Proposition}
\newtheorem{assumption}[theorem]{Assumption}
\newtheorem{lemma}[theorem]{Lemma}
\newtheorem{informal theorem}[theorem]{Informal Theorem}
\newtheorem{fact}[theorem]{Fact}

\newtheorem{observation}[theorem]{Observation}

\theoremstyle{definition}
\newtheorem{defn}{Definition}[section]

\newcommand{\para}[1]{\smallskip {\bf #1} \/}

\newcommand{\N}{\mathbb{N}}
\newcommand{\hw}{\hat{w}}
\newcommand{\R}{\mathbb{R}}
\newcommand{\tS}{\tilde{S}}

\usepackage{xspace}
\newcommand{\LDP}{\texttt{LabelDP}\xspace}

\newcommand{\bD}{\mathbf{D}}

\usepackage{url}
\newcommand{\SM}{Supplementary Material\xspace}

\newcommand{\papertitle}{Deep Learning with Label Differential Privacy}
\title{\papertitle}

\author{
  Badih Ghazi\\
  Google Research\\
  \texttt{badihghazi@google.com} \\
  \And 
  Noah Golowich\thanks{Part of this work was done while at Google Research.} \\
  EECS, MIT \\
  \texttt{nzg@mit.edu}
  \And 
  Ravi Kumar\\
  Google Research\\
  \texttt{ravi.k53@gmail.com}\\
  \AND
  Pasin Manurangsi\\
  Google Research\\
  \texttt{pasin@google.com}\\
  \And
  Chiyuan Zhang\\
  Google Research\\
  \texttt{chiyuan@google.com}
}

\begin{document}

\maketitle

\begin{abstract}
The Randomized Response (\RRShort) algorithm \citep{warner1965randomized} is a classical technique to improve robustness in survey aggregation, and has been widely adopted in applications with differential privacy guarantees. We propose a novel algorithm, \emph{Randomized Response with Prior} (\RRP), which can provide more accurate results while maintaining the same level of privacy guaranteed by \RRShort. We then apply \RRP to learn neural networks with \emph{label} differential privacy (\LDP), and show that when only the label needs to be protected, the model performance can be significantly improved over the previous state-of-the-art private baselines. Moreover, we study different ways to obtain priors, which when used with \RRP can  additionally improve the model performance, further reducing the accuracy gap between private and non-private models. We complement the empirical results with theoretical analysis showing that \LDP is provably easier than protecting both the inputs and labels.
\end{abstract}

\section{Introduction}

The widespread adoption of machine learning in recent years has increased the concerns about the privacy of individuals whose data is used during the model training. Differential privacy (DP)~\citep{DworkMNS06,DworkKMMN06} has emerged as a popular privacy notion that has been the basis of several practical deployments in industry~\citep{erlingsson2014rappor,CNET2014Google, greenberg2016apple, dp2017learning, ding2017collecting} and the U.S. Census~\citep{abowd2018us}.

A classical algorithm---that predates DP and was initially designed to eliminate evasive answer biases in survey aggregation---is \emph{Randomized Response} (\RRShort)~\cite{warner1965randomized}: when the input is an element from a finite alphabet, the output is equal to the input with a certain probability, and is a uniform random other element from the alphabet with the remaining probability. This simple algorithm is shown to satisfy the strong notion of \emph{local} DP~\cite{evfimievski2003limiting, kasiviswanathan2011can}, whereby the response of each user is protected, in contrast with the so-called \emph{central} DP setting where a curator has access to the raw user data, and only the output of the curator is required to be DP. We note that schemes building on \RRShort have been studied in several previous works on DP estimation (e.g., \cite{duchi2013local, kairouz2016discrete}), and have been deployed in practice \cite{CNET2014Google}.

Meanwhile, the large error incurred by \RRShort (e.g.,~\cite{chan2012optimal}) has stimulated significant research aiming to improve its accuracy, mostly by relaxing to weaker privacy models (e.g.,~\cite{erlingsson2019amplification, cheu2019distributed,acharya2020context}). In this work, we use a different approach and seek to improve \RRShort by leveraging available \emph{prior} information.
(A recent work of \citet{liu2021leveraging} also used priors to improve accuracy, but in the context of the DP \emph{multiplicative weights} algorithm, which applies in the \emph{central} DP setting.) The prior information can consist of domain-specific knowledge, (models trained on) publicly available data, or historical runs of a training algorithm. Our algorithm is presented in Section~\ref{sec:mult_stage_training_algo} (Algorithm~\ref{alg:rr-with-prior}).  At a high level, given a prior distribution $\bp$ on the alphabet, the algorithm uses $\bp$ to prune the alphabet. If the prior is reliable and even if the alphabet is only minimally pruned, the probability that the output equals the input is larger than when \RRShort is applied to the entire alphabet. On the other hand, if the prior is uniform over the entire alphabet, then our algorithm recovers the classical \RRShort. To implement the above recipe, one needs to specify how to effect the pruning using $\bp$. It turns out that the magnitude of pruning can itself vary depending on $\bp$, but we can obtain a closed-form formula for determining this.  Interestingly, by studying a suitable linear program, we show that the resulting \RRP strategy is \emph{optimal} in that among all $\epsilon$-DP algorithms, it maximizes the probability that the output equals the input when the latter is sampled from $\bp$ (Theorem~\ref{lem:rr-prior-opt}).

\subsection{Applications to Learning with Label Differential Privacy}

There have been a great number of papers over the last decade that developed DP machine learning algorithms (e.g., \citep{chaudhuri2011differentially, zhang2012functional, song2013stochastic,shokri2015privacy,Smith2018DifferentiallyPR,Smith2019DifferentiallyPR, phan2020scalable}). In the case of deep learning, the seminal work of \citet{abadi2016deep} introduced a DP training framework (\dpsgd) that was integrated into TensorFlow~\citep{tf-privacy} and PyTorch~\citep{pytorch-privacy}. Despite numerous followup works, including, e.g., \citep{papernot2016semi,papernot2018scalable,papernot2020tempered, mcmahan2017learning, yu2019differentially, nasr2020improving,9378011,tramer2020differentially}, and extensive efforts, the accuracy of models trained with \dpsgd remains significantly lower than that of models trained without DP constraints. Notably, for the widely considered CIFAR-$10$ dataset, the highest reported accuracy for DP models is $69.3\%$~\cite{tramer2020differentially}, which strikingly relies on \emph{handcrafted} visual features despite that in non-private scenarios \emph{learned} features long been shown to be superior. Even using pre-training with external (CIFAR-100) data, the best reported DP accuracy, $73\%$\footnote{For DP parameters of $\epsilon = 8$ and $\delta = 10^{-5}$, cited from \citet[][Figure~6]{abadi2016deep}. For a formal definition of DP, we refer the reader to Definition~\ref{def:dp_general}.}, is still far below the non-private baselines ($>95\%$). The performance gap becomes a roadblocker for many real-world applications to adopt DP. In this paper, we focus on a more restricted, but important, special case where the DP guarantee is only required to hold with respect to the labels, as described next.

In the \emph{label differential privacy} (\LDP) setting, the \emph{labels} are considered sensitive, and their privacy needs to be protected, while the input points are not sensitive. This notion has been studied in the PAC setting \citep{chaudhuri2011sample,beimel2013private} and for the particular case of sparse linear regression \citep{wang2019sparse}, and it captures several practical scenarios. Examples include: (i) computational advertising where the impressions are known to the Ad Tech\footnote{
Ad tech (abbreviating Advertising Technology) comprises the tools that help agencies and brands target, deliver, and analyze their digital advertising efforts; see, e.g.,
\url{blog.hubspot.com/marketing/what-is-ad-tech}.
}, and thus considered non-sensitive, while the conversions reveal user interest and are thus private (see, e.g., \citet{chrome-blog-post} and \cite{MaskedLARK21}), (ii) recommendation systems where the choices are known, e.g., to the streaming service provider, but the user ratings are considered sensitive, and (iii) user surveys and analytics where demographic information (e.g., age, gender) is non-sensitive but income
is sensitive---in fact, this was the motivating reason for Warner~\cite{warner1965randomized} to propose \RRShort many decades ago!  We present a novel \emph{multi-stage algorithm} (\multiST) for training deep neural networks with \LDP that builds on top of \RRP (see Section~\ref{sec:mult_stage_training_algo} and Algorithm~\ref{alg:lp_mst}), and we benchmark its empirical performance (Section~\ref{sec:eval}) on multiple datasets, domains, and architectures, including the following.
\begin{itemize}
    \item On CIFAR-10, we show that it achieves $20\%$ higher accuracy than \dpsgd
    \footnote{We remark that the notion of $\eps$-DP in~\cite{abadi2016deep,papernot2018scalable,papernot2016semi} is \emph{not} directly comparable to $\eps$-Label DP in our work in that they use the addition/removal notion whereas we use the substitution one. Please see the \SM for more discussion on this.}.
    \item On the more challenging CIFAR-$100$, we present the first non-trivial DP learning results.
    \item On MovieLens, which consists of user ratings of movies, we show improvements via \multiST.
\end{itemize}

In some applications, domain specific algorithms can be used to obtain priors directly without going through multi-stage training. For image classification problems, we demonstrate how priors computed from a (non-private) \emph{self-supervised learning}~\citep{chen2020simple,chen2020big,grill2020bootstrap,he2020momentum,caron2021emerging} phase on the input images can be used to achieve higher accuracy with a \LDP guarantee with extremely small privacy budgets ($\epsilon\leq 0.1$, see Section~\ref{sec:domain_specific_priors} for details).

We note that due to the requirement of \dpsgd to compute and clip \emph{per-instance gradient}, it remains technically challenging to scale to larger models or mini-batch sizes, despite numerous attempts to minigate this problem~\citep{goodfellow2015efficient,agarwal2019auto,dangel2019backpack,subramani2020enabling}. On the other hand, our formulation allows us to use state-of-the-art deep learning architectures such as ResNet~\citep{he2016deep}. We also stress that our \multiST algorithm goes beyond deep learning methods that are robust to label noise. (See~\cite{song2020learning} for a survey of the latter.)

Our empirical results suggest that protecting the privacy of labels can be significantly easier than protecting the privacy of both inputs and labels. We find further evidence to this by showing that for the special case of stochastic convex optimization (SCO), the sample complexity of algorithms privatizing the labels is much smaller than that of algorithms privatizing both labels and inputs; specifically, we achieve \emph{dimension-independent} bounds for \LDP (Section~\ref{sec:sco}). We also show that a good prior can ensure smaller population error for non-convex loss. (Details are in the Supplementary Material.)

\section{Preliminaries}
\label{sec:prelim}

For any positive integer $K$, let $[K] := \{1,\dots, K\}$. 
\emph{Randomized response} (\RRShort)~\citep{warner1965randomized} is the following: let $\epsilon\geq 0$ be a parameter and let $y\in[K]$ be the true value known to $\text{\RRShort}_\epsilon$.  When an \emph{observer} queries the value of $y$, $\text{\RRShort}_\epsilon$ responds with a random draw $\ty$ from the following probability distribution:
\begin{align}
\label{eq:rr-labeldp}
\Pr[\ty = \hat{y}] =
\begin{cases}
\frac{e^{\eps}}{e^{\eps} + K - 1} & \text{ for } \hat{y} = y, \\
\frac{1}{e^{\eps} + K - 1} & \text{ otherwise}.
\end{cases}
\end{align}
In this paper, we focus on the application of learning with label differential privacy. We recall the definition of differential privacy (DP), which is applicable to any notion of \emph{neighboring datasets}. For a textbook reference, we refer the reader to \citet{dwork2014algorithmic}. 
\begin{defn}[Differential Privacy (DP)~\cite{DworkKMMN06,DworkMNS06}]\label{def:dp_general}
Let $\epsilon, \delta \in \R_{\geq 0}$. A randomized algorithm $\AShort$ taking as input a dataset is said to be \emph{$(\epsilon, \delta)$-differentially private} ($(\epsilon, \delta)$-DP) if for any two \emph{neighboring datasets} $\bD$ and $\bD'$, and for any subset $S$ of outputs of $\AShort$, it is the case that $\Pr[\AShort(\bD) \in S] \le e^{\epsilon} \cdot \Pr[\AShort(\bD') \in S] + \delta$. If $\delta = 0$, then $\AShort$ is said to be \emph{$\epsilon$-differentially private} ($\epsilon$-DP).
\end{defn}

When applied to machine learning methods in general and deep learning in particular, DP is usually enforced on the weights of the trained model~\citep[see, e.g.,][]{chaudhuri2011differentially,kifer_private_2012,abadi2016deep}. In this work, we focus on the notion of label differential privacy.
\begin{defn}[Label Differential Privacy]\label{def:ldp_general}
Let $\epsilon, \delta \in \R_{\geq 0}$. A randomized training algorithm $\AShort$ taking as input a dataset is said to be \emph{$(\epsilon, \delta)$-label differentially private} ($(\epsilon, \delta)$-\LDP) if for any two training datasets $\bD$ and $\bD'$ that differ in the \emph{label} of a \emph{single example}, and for any subset $S$ of outputs of $\AShort$, it is the case that $\Pr[\AShort(\bD) \in S] \le e^{\epsilon} \cdot \Pr[\AShort(\bD') \in S] + \delta$.  If $\delta = 0$, then $\AShort$ is said to be \emph{$\epsilon$-label differentially private} ($\epsilon$-\LDP).
\end{defn}

All proofs skipped in the main body are given in the \SM.

\section{Randomized Response with Prior}\label{sec:mult_stage_training_algo}

In many real world applications, a prior distribution about the labels could be publicly obtained from domain knowledge and help the learning process. In particular, we consider a setting where for each (private) label $y$ in the training set, there is an associated prior $\bp = (p_1, \dots, p_K)$. The goal is to output a randomized label $\ty$ that maximizes the probability that the output is correct (or equivalently maximizes the signal-to-noise ratio), i.e., $\Pr[y = \ty]$. The privacy constraint here is that the algorithm should be $\epsilon$-DP with respect to $y$. (It need \emph{not} be private with respect to the prior $\bp$.)

We first describe our algorithm \RRP by assuming access to such priors. 

\subsection{Algorithm: \RRP}
\label{subsec:rr-with-prior}

We build our \RRP algorithm with a subroutine called \RRtopk, as shown in Algorithm~\ref{alg:rr-top-k}, which is a modification of randomized response where we only consider the set of $k$ labels $i$ with largest $p_i$. Then, if the input label $y$ belongs to this set, we use standard randomized response on this set. Otherwise, we output a label from this set uniformly at random.

\begin{algorithm}[!htp]
  \caption{\bf \RRtopk \label{alg:rr-top-k}}
  \textbf{Input:} A label $y \in [K]$ \\
  \textbf{Parameters:} $k \in [K]$, prior $\bp = (p_1, \dots, p_K)$
  \begin{enumerate}[leftmargin=14pt,rightmargin=20pt,itemsep=1pt,topsep=1.5pt]
  \item Let $Y_k$ be the set of $k$ labels with maximum prior probability (with ties broken arbritrarily).
  \item If $y \in Y_k$, then output $y$ with probability $\frac{e^{\eps}}{e^{\eps} + k - 1}$ and output $y' \in Y_k \setminus \{y\}$ with probability $\frac{1}{e^{\eps} + k - 1}$.
  \item If $y \not\in Y_k$, output an element from $Y_k$ uniformly at random.
  \end{enumerate}
 \end{algorithm}

The main idea behind \RRP is to dynamically estimate an optimal $k^*$ based on the prior $\bp$, and run \RRtopk with $k^*$. Specifically, we choose $k^*$ by maximizing $\Pr[\RRtopk(y) = y]$. It is not hard to see that this expression is exactly equal to $\frac{e^{\eps}}{e^{\eps} + k - 1} \cdot \left(\sum_{\ty \in Y_k} p_{\ty}\right)$ if $y \sim \mathbf{p}$. \RRP is presented in Algorithm~\ref{alg:rr-with-prior}.

\begin{algorithm}[!htp]
  \caption{\bf \RRP \label{alg:rr-with-prior}}
  \textbf{Input:} A label $y \in [K]$ \\
  \textbf{Parameters:} prior $\bp = (p_1, \dots, p_K)$
  \begin{enumerate}[leftmargin=14pt,rightmargin=20pt,itemsep=1pt,topsep=1.5pt]
  \item For $k \in [K]$:
  \begin{enumerate}
    \item Compute $w_k := \frac{e^{\eps}}{e^{\eps} + k - 1} \cdot \left(\sum_{\ty \in Y_k} p_{\ty}\right)$, where $Y_k$ is the set of $k$ labels with maximum prior probability (ties broken arbritrarily).
  \end{enumerate}
  \item Let $k^* = \argmax_{k \in [K]} w_k$. \label{step:compute-opt-k}
  \item Return an output of \RRtopk($y$) with $k = k^*$.
  \end{enumerate}
 \end{algorithm}
 
\subsubsection{Privacy Analysis}
It is not hard to show that \RRtopk is $\eps$-DP.

\begin{lemma} \label{lem:rr-top-k-dp}
\RRtopk is $\eps$-DP.
\end{lemma}
 
The privacy guarantee of \RRP follows immediately from that of \RRtopk (Lemma~\ref{lem:rr-top-k-dp}) since our choice of $k$ does not depend on the label $y$:
\begin{corollary} \label{cor:rr-with-prior-dp}
\RRP is $\eps$-DP.
\end{corollary}

For learning with a \LDP guarantee, we first use \RRP to query a randomized label for each example of the training set, and then apply a general learning algorithm that is robust to random label noise to this dataset. Note that unlike \dpsgd~\citep{abadi2016deep} that makes new queries on the gradients in every training epoch, we query the randomized label \emph{once} and reuse it in all the training epochs.

\subsection{Optimality of \RRP}
\label{sec:rrp-optimality}

In this section we will prove the optimality of \RRP. For this, we will need  additional notation. For any algorithm $\RShort$ that takes as input a label $y$ and outputs a randomized label $\ty$, we let $\obj_{\bp}(\RShort)$ denote the probability that the output label is equal to the input label $y$ when $y$ is distributed as $\bp$; i.e., $\obj_{\bp}(\RShort) = \Pr_{y \sim \bp}[\RShort(y) = y]$, where the distribution of $y \sim \bp$ is $\Pr[y = i] = p_i$ for all $i \in [K]$.

The main result of this section is that, among all $\eps$-DP algorithms, \RRP maximizes $\obj_{\bp}(\RShort)$, as stated more formally next.

\begin{theorem} \label{lem:rr-prior-opt}
Let $\bp$ be any probability distribution on $[K]$ and $\RShort$ be any $\eps$-DP algorithm that randomizes the input label given the prior $\bp$ . We have that
\begin{align*}
\obj_{\bp}(\RRP) \geq \obj_{\bp}(\RShort).
\end{align*}
\end{theorem}

Before we proceed to the proof, we remark that our proof employs a linear program (LP) to characterize the optimal mechanisms; a generic form of such LPs has been used before in~\cite{HardtT10,GhoshRS12}. However, these works focus on different problems (linear queries) and their results do not apply here.

\begin{proof}[Proof of \Cref{lem:rr-prior-opt}]
Consider any $\eps$-DP algorithm $\RShort$, and let  $q_{\ty|y}$ denote $\Pr[\RShort(y) = \ty]$. Observe that $\obj_{\bp}(\RShort) = \sum_{y \in [k]} p_y \cdot q_{y \mid y}$.

Since $q_{\cdot | y}$ is a probability distribution, we must have that
\begin{align*}
\sum_{\ty \in [K]} q_{\ty | y} = 1, \forall y \in [K], \mbox{ and }
q_{\ty | y} \geq 0, \forall \ty, y \in [K].
\end{align*}
Finally, the $\eps$-DP guarantee of $\RShort$ implies that
\begin{align*}
q_{\ty | y} \leq e^{\eps} \cdot q_{\ty | y'} & &\forall \ty, y, y' \in [K].
\end{align*}
Combining the above, $\obj_{\bp}(\RShort)$ is upper-bounded by the optimum of the following linear program (LP), which we refer to as LP1:
\begin{align}
&\max &\sum_{y \in [k]} p_y \cdot q_{y|y} & & \nonumber \\
&\text{s.t.} &q_{\ty | y} \leq e^{\eps} \cdot q_{\ty | y'} & &\forall \ty, y, y' \in [K], \label{eq:dp-constraint-1} \\
& &\sum_{\ty \in [K]} q_{\ty | y} = 1 & &\forall y \in [K], \label{eq:probability-mass-1} \\
& &q_{\ty | y} \geq 0 & & \forall \ty, y \in [K]. \nonumber
\end{align}

Notice that constraints~\eqref{eq:dp-constraint-1} and~\eqref{eq:probability-mass-1} together imply that:
\begin{align*}
q_{y | y} + e^{-\eps} \cdot \sum_{\ty \in [K] \setminus \{y\}} q_{\ty | \ty} \leq 1 & &\forall y \in [K].
\end{align*}
In other words, the optimum of LP1 is at most the optimum of the following LP that we call LP2:
\begin{align}
&\max &\sum_{y \in [k]} p_y \cdot q_{y|y} & & \nonumber \\
&\text{s.t.} &q_{y | y} + e^{-\eps} \cdot \sum_{\ty \in [K] \setminus \{y\}} q_{\ty | \ty} \leq 1 & &\forall y \in [K], \label{eq:probability-mass-2} \\
& &q_{y | y} \geq 0 & & \forall y \in [K]. \label{eq:non-negative}
\end{align}

An optimal solution to LP2 must be a vertex (aka extreme point) of the polytope defined by~\eqref{eq:probability-mass-2} and~\eqref{eq:non-negative}. Recall that an extreme point of a $K$-dimensional polytope must satisfy $K$ independent constraints with equality. In our case, this means that one of the following occurs:
\begin{itemize}[nosep]
\item Inequality \eqref{eq:non-negative} is satisfied with equality for all $y \in [K]$ resulting in the all-zero solution (whose objective is zero), or,
\item For some non-empty subset $Y \subseteq [K]$, inequality \eqref{eq:probability-mass-2} is satisfied with equality for all $y \in Y$, and inequality ~\eqref{eq:non-negative} is satisfied with equality for all $y \in [K] \setminus Y$. This results in 
\begin{align*}
q_{y|y} =
\begin{cases}
\frac{e^{\eps}}{e^{\eps} + |Y| - 1} & \text{ if } y \in Y, \\
0 & \text{ if } y \notin Y.
\end{cases}
\end{align*}
This yields an objective value of $\frac{e^{\eps}}{e^{\eps} + |Y| - 1} \cdot \sum_{y \in Y} p_y$.
\end{itemize}
In conclusion, we have that
\begin{align*}
\obj_{\bp}(\RShort) &\leq \max_{\emptyset \ne Y \subseteq [K]} \frac{e^{\eps}}{e^{\eps} + |Y| - 1} \cdot \sum_{y \in Y} p_y \\
&= \max_{k \in [K]} \frac{e^{\eps}}{e^{\eps} + k - 1} \cdot \max_{Y \subseteq [K], |Y| = k} \sum_{y \in Y} p_y \\
&= \max_{k \in [K]} \frac{e^{\eps}}{e^{\eps} + k - 1} \cdot \sum_{y \in Y_k} p_y 
\enspace = \enspace \max_{k \in [K]} w_k,
\end{align*}
where the last two equalities follow from our definitions of $Y_k$ and $w_k$. Notice that $\obj_{\bp}(\RRP) = \max_{k \in [K]} w_k$. Thus, we get that $\obj_{\bp}(\RRP) \geq \obj_{\bp}(\RShort)$ as desired.
\end{proof}

\section{Application of $\RRP$: Multi-Stage Training}
\label{subsec:multi-stage}

Our \RRP algorithm requires publicly available priors, which could usually be obtained from domain specific knowledge. In this section, we describe a training framework that bootstraps from a uniform prior, and progressively learns refined priors via multi-stage training. This general framework can be applied to arbitrary domains even when no public prior distributions are available.

Specifically, we assume that we have a training algorithm $\AShort$ that outputs a probabilistic classifier which, on a given unlabeled sample $\bx$, can assign a probability $p_y$ to each class $y \in [K]$.  We partition our dataset into subsets $S^{(1)}, \dots, S^{(T)}$, and we start with a trivial model $M^{(0)}$ that outputs equal probabilities for all classes. At each stage $t \in [T]$, we use the most recent model $M^{(t - 1)}$ to assign the probabilities $(p_1, \dots, p_K)$ for each sample $\bx_i$ from $S^{(t)}$. Applying \RRP\ with this prior on the true label $y_i$, we get a randomized label $\ty_i$ for $\bx_i$. We then use all the samples with randomized labels obtained so far to train the model $M^{(t)}$.

The full description of our \multiST (Label Privacy Multi-Stage Training) method is presented in Algorithm~\ref{alg:lp_mst}. We remark here that the partition $S^{(1)}, \dots, S^{(T)}$ can be arbitrarily chosen, as long as it does not depend on the labels $y_1, \dots, y_n$. We also stress that the training algorithm $\AShort$ need \emph{not} be private.  We use \oneST to denote our algorithm with one stage, \twoST to denote our algorithm with two stages, and so on. We also note that \oneST is equivalent to using vanilla \RRShort.
The $t$th stage of a multi-stage algorithm is denoted \emph{stage-$t$}. 

\begin{algorithm}[!htp]
\small
  \caption{\bf \MultiStage (\multiST) \label{alg:lp_mst}}
  \textbf{Input:} Dataset $S = \{(\bx_1, y_1), \ldots, (\bx_n, y_n)\}$ \\
  \textbf{Parameters:} Number $T$ of stages, training algorithm $\AShort$
  \begin{enumerate}[leftmargin=14pt,rightmargin=20pt,itemsep=1pt,topsep=1.5pt]
  \item Partition $S$ into $S^{(1)}, \dots, S^{(T)}$
  \item Let $M^{(0)}$ be the trivial model that always assigns equal probability to each class.
  \item For $t = 1$ to $T$:
    \begin{enumerate}
    \item Let $\tS^{(t)} = \emptyset$.
    \item For each $(\bx_i, y_i) \in S^{(t)}$:
    \begin{enumerate}
        \item Let $\bp = (p_1, \dots, p_K)$ be the probabilities predicted by $M^{(t)}$ on $\bx_i$.
        \item Let $\ty_i = \RRP_{\bp}(y_i)$. \label{step:rr-prior}
        \item Add $(\bx_i, \ty_i)$ to $\tS^{(t)}$.
    \end{enumerate}
    \item Let $M^{(t)}$ be the model resulting from training on $\tS^{(1)}\cup \cdots\cup \tS^{(t)}$ using $\AShort$.
    \end{enumerate}
  \item Output $M^{(T)}$. 
  \end{enumerate}
 \end{algorithm}

The privacy guarantee of \multiST is given by the following:
 
\begin{observation} \label{obs:multist-dp}
For any $\eps > 0$, if \RRP is $\eps$-DP, then \multiST is $\eps$-\LDP.
\end{observation}
\begin{proof}We will in fact prove a stronger statement that the algorithm is $\eps$-DP even when we output all the $T$ models $M^{(1)}, \dots, M^{(T)}$ together with all the randomized labels $\ty_1, \dots, \ty_n$. For any possible output models $m^{(1)}, \dots, m^{(T)}$ and output labels $z_1, \dots, z_n$, we have
\begin{align*}
&\Pr[M^{(1)} = m^{(1)}, \dots, M^{(T)} = m^{(T)}, \ty_1 = z_1, \dots, \ty_n = z_n] \\ 
&= \prod_{t=1}^T \left(\Pr\left[M^{(t)} = m^{(t)} \middle\vert \bigwedge_{i \in S^{(1)} \cup \cdots \cup S^{(t)}} \ty_i = z_i\right] \cdot \prod_{i \in S^{(t)}} \Pr\left[\ty_i = z_i \middle\vert M^{(t - 1)} = m^{(t - 1)} \right] \right).
\end{align*}

Consider any two datasets $\bD, \bD'$ that differ on a single user's label; suppose this user is $j$ and that the user belongs to partition $\ell \in [T]$. Then, the above expression for $\bD$ and that for $\bD'$ are the same in all but one term: $\Pr\left[\ty_j = z_j \middle\vert M^{(\ell - 1)} = m^{(\ell - 1)} \right]$, which is the probability that $\RRP_{m^{(\ell - 1)}(x_i)}$ outputs $z_i$. Since $\RRP$ is $\eps$-DP, we can conclude that the ratio between the two probabilities is at most $e^{\eps}$ as desired. 
\end{proof}

We stress that this observation holds because each sensitive label $y_i$ is only used \emph{once} in Line~\ref{step:rr-prior} of Algorithm~\ref{alg:lp_mst}, as the dataset $S$ is partitioned at the beginning of the algorithm. As a result, since each stage is $\eps$-\LDP, the entire algorithm is also $\eps$-\LDP. This is known as (an adaptive version of a) \emph{parallel composition}~\cite{McSherry10}. 

Finally, we point out that the running time of our \RRP algorithm is quasi-linear in $K$ (the time needed to sort the prior). This is essentially optimal within multistage training, since $O(K)$ time will be required to write down the prior after each stage. Moreover, for reasonable values of $K$, the running time will be dominated by back-propagation for gradient estimation. Moreover, the focus of the current work is on small to modest label spaces (i.e., values of $K$).

\section{Empirical Evaluation}
\label{sec:eval}

We evaluate \RRP on standard benchmark datasets that have been widely used in previous works on private machine learning. Specifically, in the first part, we study our general multi-stage training algorithm that boostraps from a uniform prior. We evaluate it on image classification and collaborative filtering tasks. In the second part, we focus on image classification only and use domain-specific techniques to obtain priors for \RRP. We use modern neural network architectures (e.g., ResNets~\citep{he2016deep}) and the mixup~\citep{zhang2017mixup} regularization for learning with noisy labels. Please see the \SM for full details on the datasets and the experimental setup.

\subsection{Evaluation with Multi-Stage Training}

CIFAR-10~\citep{cifar} is a 10-class image classification benchmark dataset. We evaluate our algorithm and compare it to previously reported DP baselines in Table~\ref{tab:cifar10}. Due to scalability issues, previous DP algorithms could only use simplified architectures with non-private accuracy significantly below the state-of-the-art. Moreover, even when compared to those weaker non-private baselines, a large performance drop is observed in the private models. In contrast, we use ResNet18 with 95\% non-private accuracy. Overall, our algorithms improve the previous state-of-the-art by a margin of 20\% across all $\epsilon$'s. \citet{abadi2016deep} treated CIFAR-100 as public data and use it to pre-train a representation to boost the performance of \dpsgd. We also observe performance improvements with CIFAR-100 pre-training (Table~\ref{tab:cifar10}, bottom 2 rows). But even without pre-training, our results are significantly better than \dpsgd even \emph{with} pre-training.

\begin{table*}
    \centering    \caption{Test accuracy (\%) on CIFAR-10. The baseline performances taken from previously published results
    correspond to $(\epsilon,\delta)$-DP with $\delta=10^{-5}$. The star$^\star$ indicates the use of CIFAR-$100$ pre-trained representations.}\vskip2pt
    \label{tab:cifar10}
    \begin{tabularx}{\linewidth}{Xccccccc}
    \toprule
    Algorithm & $\epsilon=1$ & $\epsilon=2$ & $\epsilon=3$ & $\epsilon=4$ & $\epsilon=6$ & $\epsilon=8$ & $\epsilon=\infty$ \\
    \midrule
    \dpsgd w/ pre-train$^\star${~\citep{abadi2016deep}}  & & 67 & & 70 & & 73 & 80 \\
    \dpsgd{~\citep{papernot2020tempered}} & & & & & \multicolumn{2}{r}{61.6{\scriptsize ($\epsilon$=7.53)}} & 76.6 \\
    Tempered Sigmoid{~\citep{papernot2020tempered}} & & & & & \multicolumn{2}{r}{66.2{\scriptsize ($\epsilon$=7.53)}} & \\
    \citet{yu2019differentially} & & & & & \multicolumn{2}{r}{44.3{\scriptsize($\epsilon$=6.78)}\phantom{X}} &\\
    \citet{nasr2020improving} & & & 55 & & & \\
    \citet{9378011} & & & & & & 53 & \\
    ScatterNet+CNN~\citep{tramer2020differentially} & & & 69.3 & & & \\
    \midrule
    \oneST & 59.96 & 82.38 & 89.89 & 92.58 & 93.58 & 94.70 & 94.96 \\
    \twoST & 63.67 & 86.05 & 92.19 & 93.37 & 94.26 & 94.52 & - \\
    \midrule
    \oneST w/ pre-train$^\star$ & 67.64 & 83.99 & 90.24 & 92.83 & 94.02 & 94.96 & 95.25 \\
    \twoST w/ pre-train$^\star$ & 70.16 & 87.22 & 92.12 & 93.53 & 94.41 & 94.59 & - \\
    \bottomrule
    \end{tabularx}
	\vskip-6pt
\end{table*}

\begin{table}
    \centering
    \caption{Experiments on CIFAR-100. The non-private baseline ($\epsilon=\infty$) is 76.38\% test accuracy.}\vskip2pt
    \label{tab:cifar100}
    \begin{tabular}{lccccc}
    \toprule
    \hspace{-5pt}Algorithm & $\epsilon=3$ & $\epsilon=4$ & $\epsilon=5$ & $\epsilon=6$ & $\epsilon=8$ \\
    \midrule
    \hspace{-5pt}{\oneST}\hspace{-8pt} & 20.96 & 46.28 & 61.38 & 68.34 & 73.59 \\
    \hspace{-5pt}{\twoST}\hspace{-8pt} & 28.74 & 50.15 & 63.51 & 70.58 & 74.14 \\
    \bottomrule
    \end{tabular}\vskip-5pt
\end{table}

In Table~\ref{tab:cifar100} we also show results on CIFAR-100, which is a more challenging variant with 10$\times$ more classes. To the best of our knowledge, these are the first non-trivial reported results on CIFAR-100 for DP learning. For $\epsilon=8$, our algorithm is only 2\% below the non-private baseline.

In addition, we also evaluate on MovieLens-1M~\citep{MovieLens}, which contains $1$ million anonymous ratings of approximately $3,900$ movies, made by 6,040 MovieLens users. Following~\citep{gaussiandp}, we randomly split the data into $80\%$ train and $20\%$ test, and show the test Root Mean Square Error (RMSE) in Table~\ref{tab:movielens}.

\begin{table}
    \centering
    \caption{Experiments on MovieLens-1M. The numbers show the test RMSE.}
    \label{tab:movielens}
    \begin{tabular}{lccccc!{\hspace{25pt}}c}
    \toprule
    Algorithm & $\epsilon=1$ & $\epsilon=2$ & $\epsilon=3$ & $\epsilon=4$ & $\epsilon=8$ & $\epsilon=\infty$ \\
    \midrule
    \textsf{LP-1ST} & 1.122 & 0.981 & 0.902 & 0.877 & 0.867 & 0.868 \\
    \textsf{LP-2ST} & 1.034 & 0.928 & 0.891 & 0.874 & 0.865 & \\
    \multicolumn{1}{l}{Gaussian DP~\cite{gaussiandp}} & & & \multicolumn{4}{r}{0.915 ({\scriptsize $\epsilon \geq 10$})\phantom{Xi}} \\
    \bottomrule
    \end{tabular}
\end{table}

Results on MNIST~\citep{mnist}, Fashion MNIST~\citep{fashionmnist}, and KMNIST~\citep{kmnist}, and comparison to more baselines can be found in the \SM. In all the datasets we evaluated, our algorithms not only significantly outperform the previous methods, but also 
greatly shrink the performance gap between private and non-private models. The latter is critical for applications of deep learning systems in real-world tasks with privacy concerns.

\para{Beyond Two Stages.}
In Figure~\ref{fig:ssl-results}(a), we report results on \multiST with $T>2$. For the cases we tried, we consistently observe 1--2\% improvements on test accuracy when going from \twoST to \threeST.  In our preliminary experiments, going beyond $T > 4$ stages leads to diminishing returns on some datasets.

\subsection{Evaluation with Domain-Specific Priors}\label{sec:domain_specific_priors}

The multi-stage training framework evaluated in the previous section is a general domain-agnostic algorithm that bootstraps itself from uniform priors. In some cases, domain-specific priors can be obtained to further improve the learning performance. In this section, we focus on image classification applications, where new advances in self-supervised learning (SSL)~\citep{chen2020simple,chen2020big,grill2020bootstrap,he2020momentum,caron2021emerging} show that high-quality image representations could be learned on large image datasets without using the class labels. In the setting of \LDP, the unlabeled images are considered public data, so we design an algorithm to use SSL to obtain priors, which is then fed to \RRP for discriminative learning.

Specifically, we partition the training examples into groups by clustering using their representations extracted from SSL models. We then query a histogram of labels for each group via discrete Laplace mechanism (aka Geometric Mechanism)~\cite{GhoshRS12}. If the groups are largely homogeneous, consisting of mostly examples from the same class, then we can make the histogram queries with minimum privacy budget. The queried histograms are used as label priors for all the points in the group. Figure~\ref{fig:ssl-results}(b) shows the results on two different SSL representations: BYOL~\citep{grill2020bootstrap}, trained on unlabeled CIFAR-10 images and DINO~\citep{caron2021emerging}, trained on ImageNet~\citep{deng2009imagenet} images. Comparing to the baseline, the SSL-based priors significantly improves the model performance with small privacy budgets. Note that since the SSL priors are not \emph{true} priors, with large privacy budget ($\epsilon=8$), it actually underperforms the uniform prior. But in most real world applications, small $\epsilon$'s are generally more useful.

\begin{figure}
    \centering
    \begin{overpic}[width=.49\linewidth]{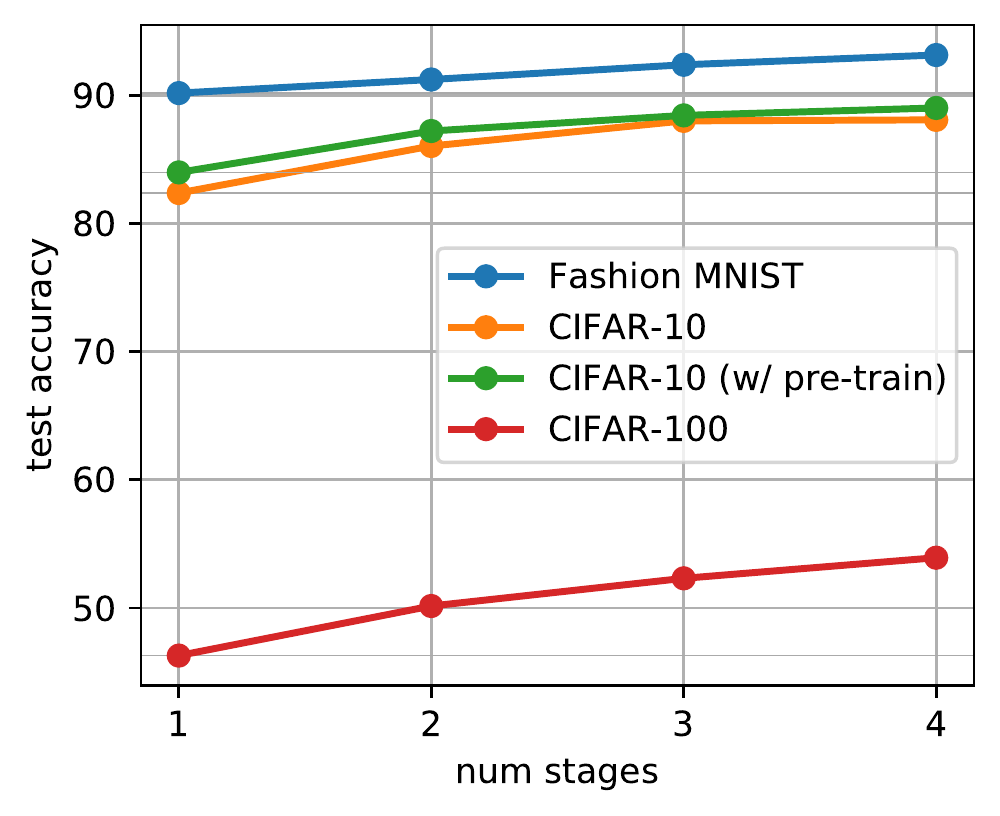}
    \put(8,1){\textbf{(a)}}
    \end{overpic}
    \begin{overpic}[width=.49\linewidth]{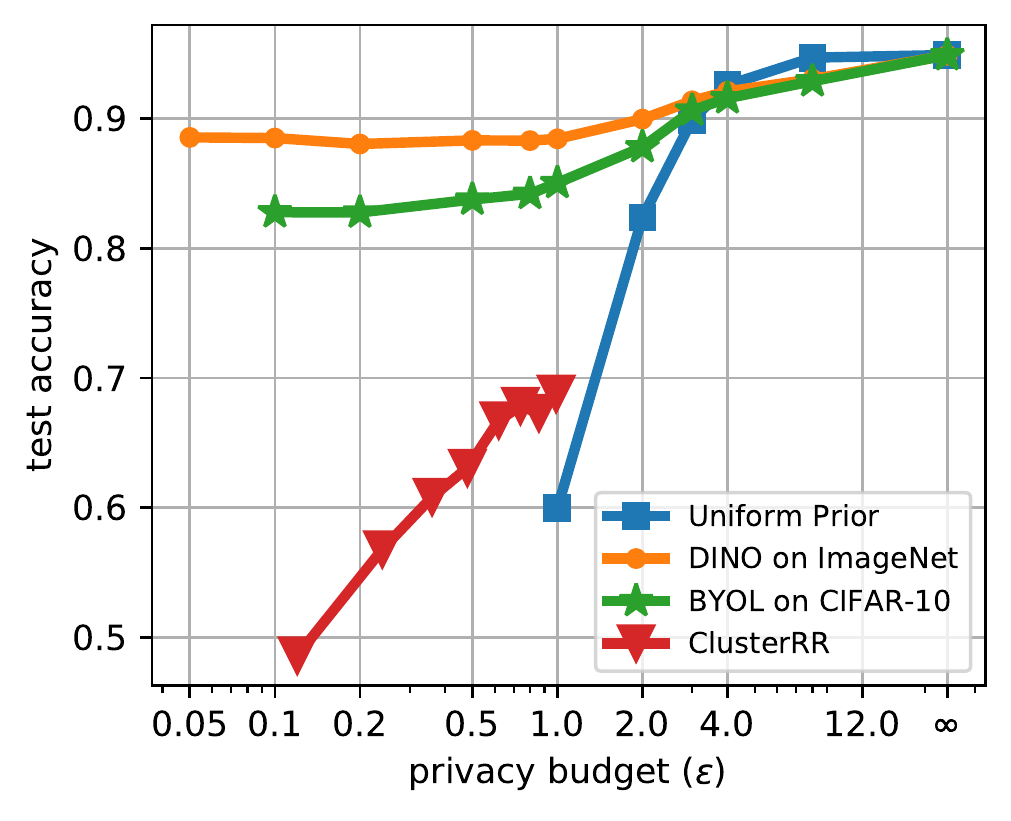}
    \put(8,1){\textbf{(b)}}
    \end{overpic}
    \caption{(a) Test accuracy (\%) on various datasets with \multiST for $T > 2$. The curve ``CIFAR-10 w/ pre-train'' is using CIFAR-100 as public data to pre-train the model. (b) \RRP with priors obtained from histogram query based on clustering in various SSL representations. We also plot recent results (ClusterRR) from \citet{esfandiari2021label}, which is a clustering based \LDP algorithm.}
    \label{fig:ssl-results}
\end{figure}

\section{Theoretical Analysis}
\label{sec:sco}

Previous works have shown that \LDP can be provably easier than DP in certain settings; specifically, in the PAC learning setting, Beimel et al.~\cite{beimel2013private} proved that finite VC dimension implies learnability by \LDP algorithms, whereas it is known that this is not sufficient for DP algorithms~\cite{AlonLMM19}.

We extend the theoretical understanding of this phenomenon to the stochastic convex optimization (SCO) setting. Specifically, we show that, by applying \RRShort on the labels and running SGD on top of the resulting noisy dataset with an appropriate debiasing of the noise, one can arrive at the following \emph{dimension-independent} excess population loss.

\begin{theorem}[Informal]
  \label{prop:labelsgd-informal}
  For any $\ep \in (0,1)$, there is an $\ep$-\LDP algorithm for stochastic convex optimization with excess population loss $\tilde{O}\left(DL \cdot \frac{K}{\ep \sqrt{n}}\right)$ where $D$ denotes the diameter of the parameter space and $L$ denotes the Lipschitz constant of the loss function.
\end{theorem}

The above excess population loss can be compared to that of Bassily et al.~\cite{BassilyFTT19}, who gave an $(\eps, \delta)$-DP algorithm with excess population loss $O_{D, L}\left(\frac{1}{\sqrt{n}} + \frac{\sqrt{p}}{\eps n}\right)$, where $p$ denote the dimension of the parameter space; this bound is also known to be tight in the standard DP setting. The main advantage of our guarantee in \Cref{prop:labelsgd-informal} is that it is independent of the dimension $p$. Furthermore, we show that our bound is tight up to polylogarithmic factors and the dependency on the number of classes $K$.

The above result provides theoretical evidence that running \RRShort on the labels and then training on this noisy dataset can be effective. We can further extend this to the setting where, instead of running \RRShort, we run \RRtopk before running the aforementioned (debiased) SGD, although---perhaps as expected---our bound on the population loss now depends on the quality of the priors.

\begin{corollary}[Informal] \label{cor:topk-sco}
Suppose that we are given a prior $\bp_x$ for every $x$ and let $Y^x_k$ denote the set of top-$k$ labels with respect to $\bp_x$. Then, for any $\ep \in (0,1)$, there is an $\ep$-\LDP algorithm for stochastic convex optimization with excess population loss $\tilde{O}\left(DL \cdot \left(\frac{k}{\ep \sqrt{n}} + \Pr_{(x, y) \sim \MD}[y \notin Y^x_k]\right)\right)$ where $D, L$ are as defined in \Cref{prop:labelsgd-informal} and $\MD$ is the data distribution.
\end{corollary}

When our top-$k$ set is perfect (i.e., $y$ always belongs to $Y_k^x$), the bound reduces to that of \Cref{prop:labelsgd-informal}, but with the smaller $k$ instead of $K$. Moreover, the second term is, in some sense, a penalty we pay in the excess population loss for the inaccuracy of the top-$k$ prior. We defer the formal treatment and the proofs to the Supplementary Material, in which we also present additional generalization results for non-convex settings.
Note that \Cref{cor:topk-sco} is not in the exact setup we run in experiments, where we dynamically calculate an optimal $k$ for each $x$ given generic priors (via \RRP), and for which the utility is much more complicated to analyze mathematically. Nonetheless, the above corollary corroborates the intuition that a good prior helps with training. 

\section{Conclusions and Future Directions}
\label{sec:conclusion}

In this work, we introduced a novel algorithm \RRP (which can be used to improve on the traditional \RRShort mechanism), and applied it to \LDP problems. We showed that prior information can be incorporated to the randomized label querying framework while maintaining privacy constraints. We demonstrated two frameworks to apply \RRP: 
(i) a general multi-stage training algorithm \multiST that bootstraps from uniform
priors and (ii) an algorithm that build priors from clustering with SSL-based
representations. The former is general purpose and can be applied to tasks even when no domain-specific priors are available, while the latter uses a domain-specific algorithm to extract priors and performs well even with very small privacy budget. As summarized by the figure on the right, in both cases, by focusing on \LDP, our \RRP significantly improved the model performance of previous state-of-the-art DP models that aimed to protect both the inputs and outputs. We note that, following up on our work, additional results on deep learning with \LDP were obtained \cite{malek2021antipodes,yuan2021practical}.
The narrowed performance gap between private and non-private models is vital for adding DP to real world deep learning models. We nevertheless stress that our algorithms only protect the labels but not the input points, which might not constitute a sufficient privacy protection in all settings.

\begin{wrapfigure}[14]{r}{0.5\textwidth}
  \begin{center}\vskip-8pt
    \includegraphics[width=\linewidth]{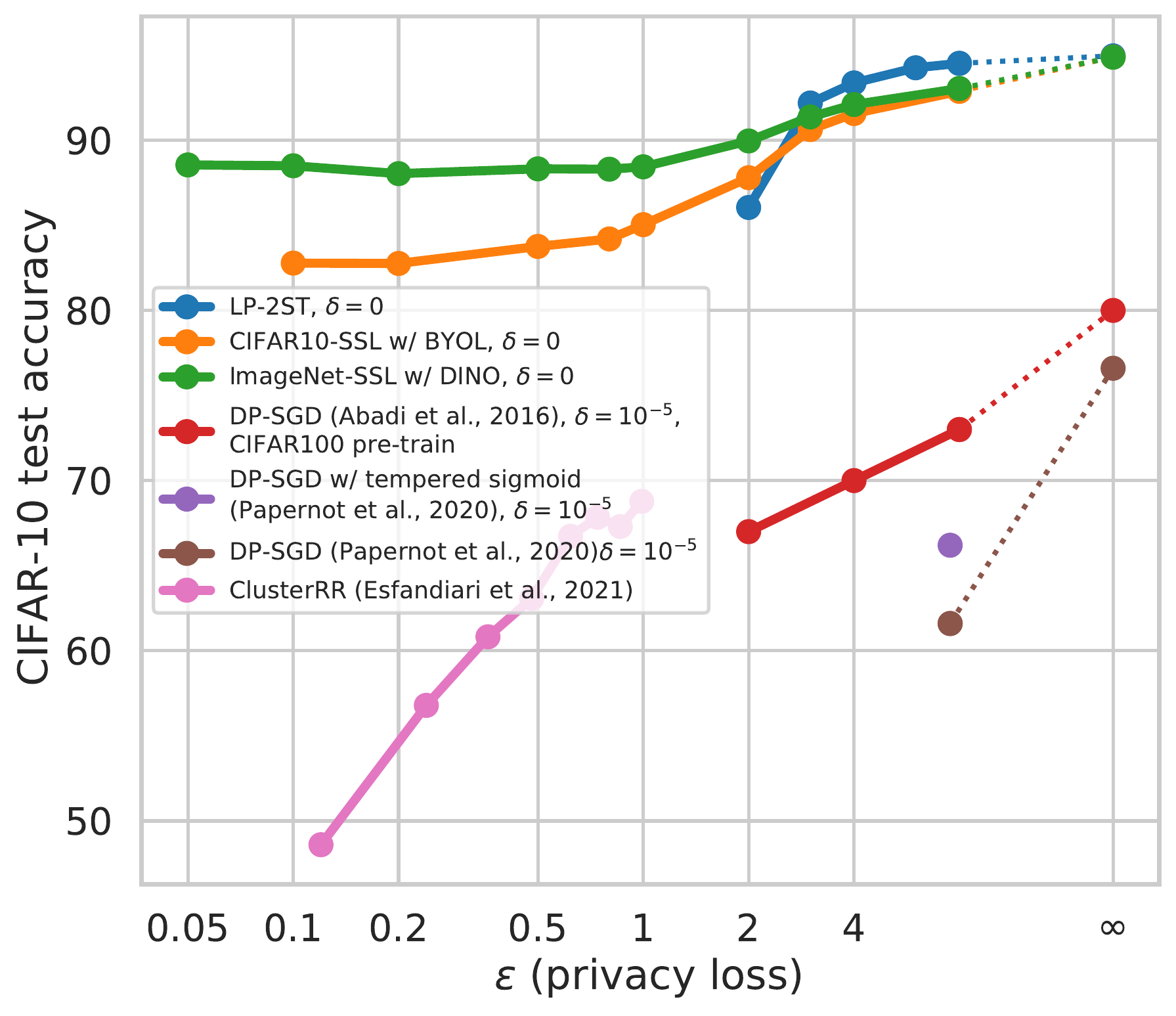}
  \end{center}
\end{wrapfigure}
Our work opens up several interesting questions.  Firstly, note that our multi-stage training procedure uses very different ingredients than those of~\citet{abadi2016deep} (which employ \dpsgd, privacy amplification by subsampling, and Renyi accounting); can these tools be used to  further improve \LDP?  Secondly, while our procedure can be implemented in the most stringent \emph{local} DP setting\footnote{They can in fact be implemented in the slightly weaker sequentially interactive local DP model \cite{duchi_minimax}.} \citep{kasiviswanathan2011can}, can it be improved in the weaker central (aka trusted curator) DP model, assuming the curator knows the prior?  Thirdly, while our algorithm achieves pure DP (i.e., $\delta = 0$), is higher accuracy possible for approximate DP (i.e., $\delta > 0$)?


\section*{Acknowledgements}
The authors would like to thank Sami Torbey for very helpful feedback on an early version of this work. At MIT, Noah Golowich was supported by a Fannie and John Hertz Foundation Fellowship and an NSF Graduate Fellowship.

\bibliographystyle{abbrvnat}
\bibliography{refs}

\appendix\clearpage

\vbox{\hsize\textwidth
\linewidth\hsize
\vskip 0.1in
\hrule height 4pt
\vskip 0.25in
\vskip -\parskip\centering
{\LARGE\bf Supplementary Material for\\``\papertitle''\par}
\vskip 0.29in
\vskip -\parskip
\hrule height 1pt
\vskip 0.09in}

\section{Missing Proofs}

\subsection{Proof of Lemma~\ref{lem:rr-top-k-dp}}

\begin{proof}[Proof of Lemma~\ref{lem:rr-top-k-dp}]
Consider any inputs $y, y' \in [K]$ and any possible output $\ty \in Y_k$. $\Pr[\RRtopk(y) = \ty]$ is maximized when $y = \ty$, whereas $\Pr[\RRtopk(y') = \ty]$ is minimized when $y' \in Y_k \setminus \{\ty\}$. This implies that
\begin{align*}
\frac{\Pr[\RRtopk(y) = \ty]}{\Pr[\RRtopk(y') = \ty]} \leq \frac{\frac{e^{\eps}}{e^{\eps} + k - 1}}{\frac{1}{e^{\eps} + k - 1}} = e^{\eps}.
\end{align*}
Thus, \RRtopk is $\eps$-DP as desired.
\end{proof}
\section{Details of the Experimental Setup}
\label{sec:exp-details}

\paragraph{Datasets.} We evaluate our algorithms on the following image classification datasets:
\begin{itemize}
    \item MNIST~\citep{mnist}, 10 class classification of hand written digits, based on inputs of $28\times 28$ gray scale images. The training set contains 60,000 examples and the test set contains 10,000.
    \item Fashion MNIST~\citep{fashionmnist}, 10 class classification of Zalando's article images. The dataset size and input format are the same as MNIST.
    \item KMNIST~\citep{kmnist}, 10 class classification of Hiragana characters. The dataset size and the input format are the same as MNIST.
    \item CIFAR-10/CIFAR-100~\citep{cifar} are 10 class and 100 class image classification datasets, respectively. Both datasets contains $32\times 32$ color images, and both have a training set of size 50,000 and a test set of size 10,000.
    \item MovieLens~\citep{MovieLens} contains a set of movie ratings from the MovieLens users. It was collected and maintained by a research group (GroupLens) at the University of Minnesota. There are 5 versions: ``25m'', ``latest-small'', ``100k'', ``1m'', ``20m''. Following \citet{gaussiandp}, we use the ``1m'' version, which the largest MovieLens dataset that contains demographic data. Specifically, it contains 1,000,209 anonymous ratings of approximately 3,900 movies made by 6,040 MovieLens users, with some meta data such as gender and zip code.
\end{itemize}

\paragraph{Architectures.} On CIFAR-10/CIFAR-100, we use ResNet~\citep{he2016deep}, which is a Residual Network architecture widely used in the computer vision community. In particular, we use ResNet18 V2~\citep{he2016identity}. Note the standard ResNet18 is originally designed for ImageNet scale (image size $224\times 224$). When adapting to CIFAR (image size $32\times 32$), we replace the initial block with $7\times 7$ convolution and $3\times 3$ max pooling with a single $3\times 3$ convolution (with stride 1) layer. The upper layers are kept the same as the standard ImageNet ResNet18.
On MNIST, Fashion MNIST, and KMNIST, we use a simplified Inception~\citep{szegedy2015going} model suitable for small image sizes, and defined as follows:
 
\begin{tabularx}{\linewidth}{rX}
 Inception :: & Conv(3$\times$3, 96) $\rightarrow$ S1 $\rightarrow$ S2 $\rightarrow$ S3 $\rightarrow$ GlobalMaxPool $\rightarrow$ Linear.\\
 S1 ::  & Block(32, 32) $\rightarrow$ Block(32, 48) $\rightarrow$ Conv(3$\times$3, 160, Stride=2). \\
 S2 ::  & Block(112, 48) $\rightarrow$ Block(96, 64) $\rightarrow$ Block(80, 80) $\rightarrow$ Block (48, 96) $\rightarrow$ Conv(3$\times$3, 240, Stride=2).\\
 S3 :: & Block(176, 160) $\rightarrow$ Block(176, 160).\\
 Block($C_1$, $C_2$) :: &  Concat(Conv(1$\times$1, $C_1$), Conv(3$\times$3,$C_2$)). \\
 Conv ::  & Convolution $\rightarrow$ BatchNormalization $\rightarrow$ ReLU.\\
\end{tabularx}

For the MovieLens experiment, we adopt a two branch neural networks from the neural collaborative filtering algorithm~\citep{he2017neural}. We simply treat the ratings as categorical labels and apply our algorithm for multi-class classification. During evaluation, we output the average rating according to the softmax probabilities output by the trained model. 

\paragraph{Training Procedures.} On MNIST, Fashion MNIST, and KMNIST, we train the models with mini-batch SGD with batch size 265 and momentum 0.9. We run the training for 40 epochs (for multi-stage training, each stage will run 40 epochs separately), and schedule the learning rate to linearly grow from 0 to 0.02 in the first 15\% training iterations, and then linearly decay to 0 in the remaining iterations. 

On CIFAR-10, we use batch size 512 and momentum 0.9, and train for 200 epochs. The learning rate is scheduled according to the widely used \emph{piecewise constant with linear rampup} scheme. Specifically, it grows from 0 to 0.4 in the first 15\% training iterations, then it remains piecewise constant with a decay factor of 10 at the 30\%, 60\%, and 90\% training iterations, respectively. The CIFAR-100 setup is similar to CIFAR-10 except that we use a batch size 256 and a peak learning rate 0.2. MovieLens experiments are trained similarly, but with batch size 128.

On all datasets, we optimize the cross entropy loss with an $\ell_2$ regularization (coefficient $10^{-4}$). All the networks are randomly initialized at the beginning of the training. For the experiment on CIFAR-10 where we explicitly study the effect of pre-training to compare with previous methods that use the same technique, we train a (non-private) ResNet18 on the full CIFAR-100 training set and initialize the CIFAR-10 model with the pre-trained weights. The classifier is still randomly initialized because there is no clear correspondence between the 100 classes of CIFAR-100 and the 10 classes of CIFAR-10. The remaining configuration remains the same as in the experiments without pre-training. In particular, we did \emph{not} freeze the pre-trained weights.

We apply standard data augmentations, including random crop, random left-right flip, and random cutout~\citep{devries2017improved}, to all the datasets during training. We implement our algorithms in TensorFlow~\citep{tensorflow2015-whitepaper}, and train all the models 
on NVidia Tesla P100 GPUs. 

\para{Learning with Noisy Labels.} Standard training procedures tend to overfit to the label noise and generalize poorly on the test set when some of the training labels are randomly flipped. We apply \emph{mixup}~\citep{zhang2017mixup} regularization, which generates random convex combinations of both the inputs and the (one-hot encoded) labels during training. It is shown that mixup is resistant to random label noise. Note that our framework is generic and in principle any robust training technique could be used. We have chosen mixup for its simplicity, but there has been a rich body of recent work on deep learning methods with label noise, see, e.g., \citep{hu2019simple,han2018co,yu2019does,chen2019understanding,zhang2018generalized,nguyen2019self,menon2019can,lukasik2020does,zhengerror,jiang2020beyond,harutyunyan2020improving,han2020sigua,ma2020normalized,song2019prestopping,pleiss2020identifying,song2020learning} and the references therein. Potentially with more advanced robust training, even higher performance could be achieved.

\para{Multi-Stage Training.} There are a few implementation enhancements that we find useful for multi-stage training.  For concreteness, we discuss them for \twoST.  First, we find it helps to initialize the stage-2 training with the models trained in stage-1. This is permitted as the stage-1 model is trained on labels that are queried privately. Moreover, we can  reuse those labels queried in stage-1 and train stage-2 on a combined dataset. Although the subset of data from stage-1 is noisier, we find that it generally helps to have more data, especially when we reduce the noise of stage-1 data by using the learned prior model. Specifically, for each sample $(x,\tilde{y})$ in the stage-1 data, where $\tilde{y}$ is the private label queried in stage-1, we make a prediction on $x$ using the model trained in stage-1; if $\tilde{y}$ is not in the top $k$ predicted classes, we will exclude it from the stage-2 training. Here $k$ is simply set to the average $k$ obtained when running \RRP to query labels on the data held out for stage-2.
Similar ideas apply to training with more stages. For example, in \threeST, stage-3 training could use the model trained in stage-2 as initialization, and use it to filter the queried labels in stage-1 and stage-2 that are outside the top $k$ prediction, and then train on the combined data of all 3 stages.

\para{Priors from Self-supervised Learning.} Recent advances in self-supervised learning (SSL)~\citep{chen2020simple,chen2020big,grill2020bootstrap,he2020momentum,caron2021emerging} show that representations learned from a large collection of unlabeled but diverse images could capture useful semantic information and can be finetuned with labels to achieve classification performance on par with the state-of-the-art fully supervised learned models. We apply SSL algorithms to extract priors for image classification problems, with the procedure described in Algorithm~\ref{alg:ssl-priors}.

\begin{algorithm}[]
\small
  \caption{SSL Priors.}\label{alg:ssl-priors}
  \textbf{Input:} Training set $D=\{(x_i,y_i)\}_{i=1}^n$, cluster count $C$, privacy budget for priors $\ep_p$, trained SSL model $f_{\text{SSL}}$.
  \begin{enumerate}[leftmargin=14pt,rightmargin=20pt,itemsep=1pt,topsep=1.5pt]
  \item Initialize $P\leftarrow \nicefrac{1}{K}$ \texttt{ones}$(n, K)$ as the uniform priors.
  \item Extract SSL features $F=\{f_\text{SSL}(x_i): (x_i, y_i)\in D\}$.
  \item Run $k$-means algorithms to partition $F$ into $C$ groups.
  \item For each $c = 1$ to $C$:
  \begin{enumerate}
      \item Compute histogram of classes $H_c\in \mathbb{N}_{\geq 0}^K$ according to the labels of examples in the $c$-th group.
      \item Get a private histogram query $\tilde{H}_c\leftarrow H_c +$ \texttt{scipy.stats.dlaplace.rvs}$(\epsilon_p/2, K)$, via the discrete Laplace mechanism.
      \item Get a prior via normalization: $p_c=\text{max}(\tilde{H}_c,0) / \sum_{k=1}^K\text{max}(\tilde{H}_c[k],0)$.
      \item For each example $i$ in group $c$, assign $P[i,:]\leftarrow p_c$.
  \end{enumerate}
  \item Output $P$. 
  \end{enumerate}
\end{algorithm}

Specifically, we choose two recent SSL algorithms: BYOL~\citep{grill2020bootstrap} and DINO~\citep{caron2021emerging}. For BYOL, we train the SSL model using the (unlabeled) CIFAR-10 images only, as a demonstration without using \emph{external} data. For DINO, we use the models pre-trained on (unlabeled) ImageNet~\citep{deng2009imagenet} images. Since ImageNet is a much larger and more diverse dataset than CIFAR-10, the SSL representations are also more capable of capturing the semantic information. Note the ImageNet images are of higher resolution and resized to $224\times 224$ during training. To extract features for $32\times 32$ CIFAR-10 images, we simply upscale the images to $224\times 224$ before feeding into the trained neural network.

We choose relatively large cluster sizes so that the private histogram query is more robust to the added discrete Laplace noise. In particular, we found $C=100$ clusters for BYOL representations and $C=50$ clusters for DINO representations achieve a good balance of robustness and accuracy.  Since $\ep_p$ will be subtracted from the privacy budget for \RRP, we simply choose the smallest $\ep_p$ without causing too much deterioration of the priors. In our experiments, we set $\ep_p=0.05$ for BYOL and $\ep_p=0.025$ for DINO. Note the model accuracy could potentially be further boosted by choosing $C$ and $\ep_p$ adaptively according to the overall privacy budget. In the following, we provide a simple study to show how the interplay between $\ep_p$ and $C$ affects the accuracy of the histogram queries.

To compute an accuracy measure on the test set, we extract features using a SSL learned models on both training and test set. A $k$-means clustering algorithm is run on the joint set of training and test features. For each cluster, we apply the discrete Laplace mechanism to make a private histogram of class distributions from \emph{only the training examples} in that cluster. The class with the maximum votes are then used as predicted labels for all the \emph{test examples} in the cluster, and compared with the true test labels to calculate the accuracy. Figure~\ref{fig:ssl-kmeans-acc} shows the accuracy with the two different SSL features under different privacy budgets ($\ep$) for making the histogram queries. As expected, the accuracy is higher with smaller clusters, but at the same time sensitive to noise introduced by the Geometric Mechanism when the privacy budget is small.

\begin{figure} 
    \centering
    \begin{overpic}[width=.49\linewidth]{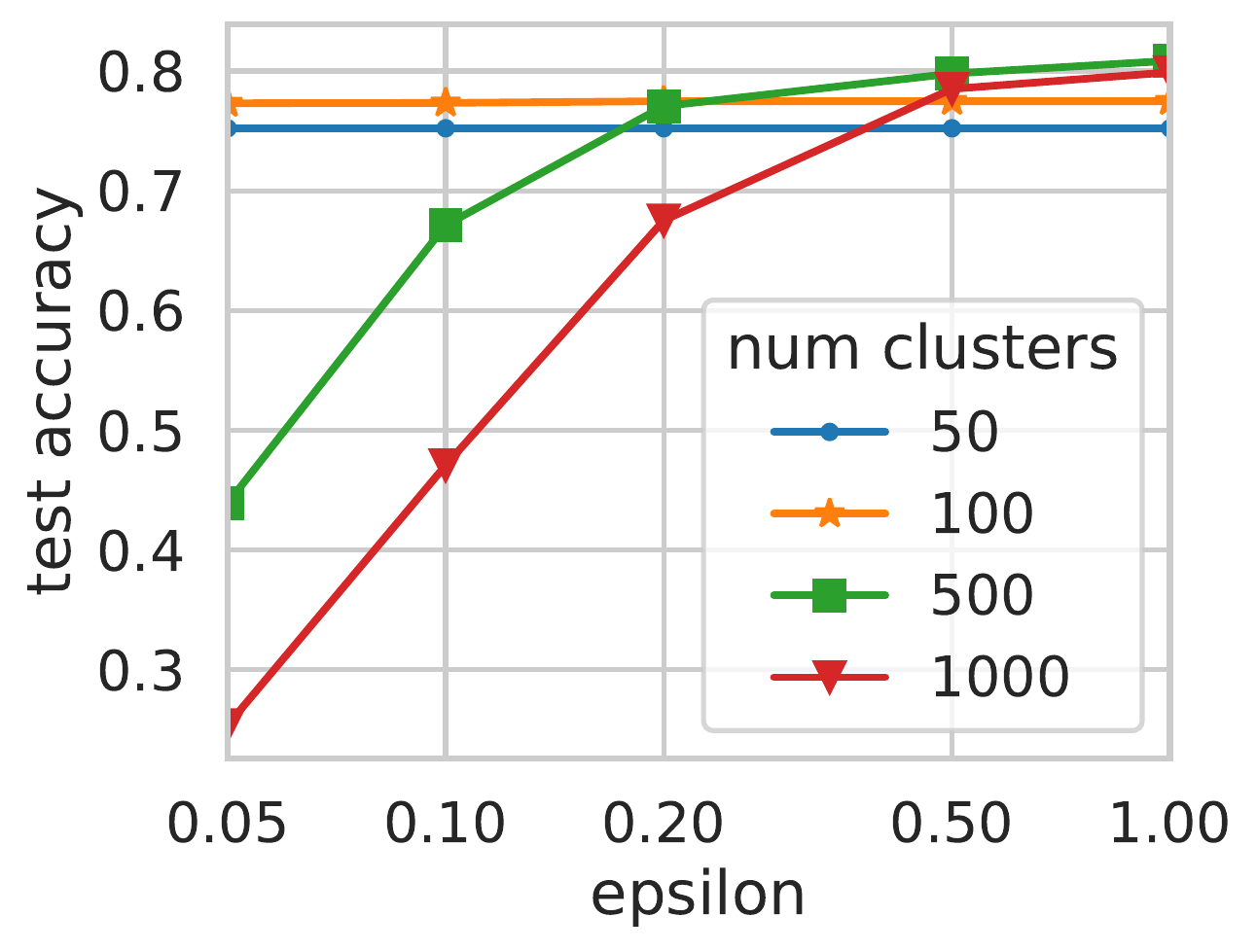}\put(1,1){\textbf{(a)} BYOL}\end{overpic}
    \begin{overpic}[width=.49\linewidth]{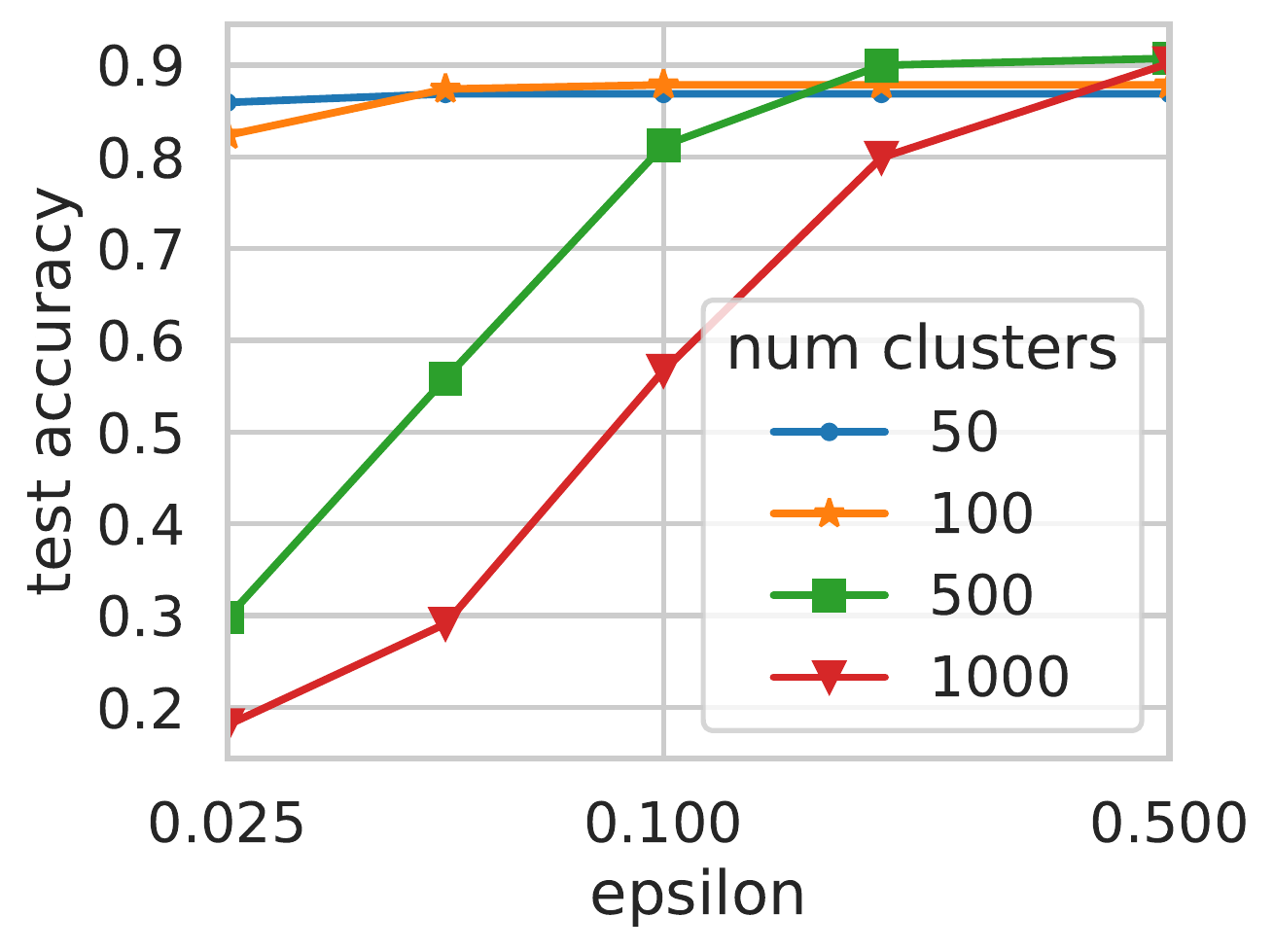}\put(1,1){\textbf{(b)} DINO}\end{overpic}
    \caption{Accuracy evaluated on CIFAR-10 test set, of private histogram querying with kmeans clustering on self-supervised learning based features learned by (a) BYOL~\citep{grill2020bootstrap} on CIFAR-10 and (b) DINO~\citep{caron2021emerging} on ImageNet.}
    \label{fig:ssl-kmeans-acc}
\end{figure}

\section{Extra Results on Multi-Stage Training}

In addition to the results presented in the main text, we include extra results of multi-stage training on  MNIST~\citep{mnist}, Fashion MNIST~\citep{fashionmnist},  and KMNIST~\citep{kmnist}. Both MNIST and Fashion MNIST have been previously used to benchmark DP deep learning algorithms. We compare our algorithms with previously reported numbers in Table~\ref{tab:mnist}. Our algorithms outperform previous methods across all $\epsilon$'s on both datasets. The gap is more pronounced on Fashion MNIST, which is slightly harder than MNIST.  Furthermore, \twoST consistently improves over \oneST. Table~\ref{tab:kmnist} shows the model performances on KMNIST under different privacy losses. The results are qualitatively similar to the ones for MNIST and Fashion MNIST.

\begin{table*}
    \centering\small\vskip-3pt
    \caption{Test accuracy (\%) on MNIST and Fashion MNIST. The baseline performances taken from previously published results
    correspond to $(\epsilon,\delta)$-DP with $\delta=10^{-5}$. 
        }\vskip2pt
    \label{tab:mnist}
    \begin{tabularx}{\linewidth}{cXcccccc}
    \toprule
    & Algorithm & $\epsilon=1$ & $\epsilon=2$ & $\epsilon=3$ & $\epsilon=4$ & $\epsilon=8$ & $\epsilon=\infty$ \\
    \midrule
    \multirow{6}{*}{\rotatebox{90}{MNIST}}
    & \dpsgd~\citep{abadi2016deep} & & 95 & & & 97 & 98.3 \\
        & PATE-G~\citep{papernot2016semi} & & \multicolumn{2}{l}{\hspace{18pt}98{\scriptsize ($\epsilon$=2.04)}} & & 98.1{\scriptsize( $\epsilon$=8.03)} & 99.2 \\
        & {Confident-GNMax}~\citep{papernot2018scalable} & & 98.5{\scriptsize ($\epsilon$=1.97)} & & & & 99.2 \\
    & Tempered Sigmoid~\citep{papernot2020tempered} & & \multicolumn{2}{l}{\hspace{30pt}98.1{\scriptsize ($\epsilon$=2.93)}} & & & \\
    & \citet{gaussiandp} & & \multicolumn{2}{l}{\hspace{21pt}96.6{\scriptsize ($\ep$=2.32)}} & \multicolumn{2}{l}{\hspace{20pt}97.0{\scriptsize ($\ep=$5.07)}} & \\
    & \citet{9378011} & & \multicolumn{2}{l}{\hspace{24pt}90.0{\scriptsize ($\ep=$2.5)}} & & & \\
    & \citet{nasr2020improving} & & & \multicolumn{2}{l}{\hspace{5pt}96.1{\scriptsize ($\ep=$3.2)}} & & \\
    & \citet{yu2019differentially} & & & & \multicolumn{2}{l}{\hspace{25pt}93.2{\scriptsize ($\ep=$6.78)}} & \\
    & \citet{feldman2020individual} & \multicolumn{2}{l}{\phantom{XX}96.56{\scriptsize ($\epsilon$=1.2)}} & 97.71 & & & \\
    \cmidrule{2-8}
    & \oneST & 95.34 & 98.16 & 98.81 & 99.08 & & 99.33 \\
    & \twoST & 95.82 & 98.78 & 99.14 & 99.24 && \\
    \midrule\midrule
    \multirow{4}{*}{\rotatebox{90}{Fashion} \rotatebox{90}{MNIST}}
    & \dpsgd~\citep{papernot2020tempered} & & \multicolumn{2}{l}{\hspace{28pt}81.9{\scriptsize ($\epsilon$=2.7)}} & & & 89.4 \\
    & {Tempered Sigmoid}~\citep{papernot2020tempered} & & \multicolumn{2}{l}{\hspace{28pt}86.1{\scriptsize ($\epsilon$=2.7)}} & & & \\
    & \citet{9378011} & & & 82.3 & & & \\
    \cmidrule{2-8}
    & \oneST & 80.78 & 90.18 & 92.52 & 93.50 & & 94.28 \\ 
    & \twoST & 83.26 & 91.24 & 93.18 & 94.10 & & \\
    \bottomrule
    \end{tabularx}\vskip-6pt
\end{table*}

\begin{table}
    \centering
    \caption{Test accuracy (\%) on KMNIST~\citep{kmnist}.
    \label{tab:kmnist}}
    \begin{tabular}{lccccc}
    \toprule
    Algorithm & $\epsilon$=1 & $\epsilon$=2 & $\epsilon$=3 & $\epsilon$=4 & $\epsilon$=$\infty$ \\
    \midrule
    \oneST & 76.56 & 92.04 & 95.86 & 96.86 & 98.33 \\
    \twoST & 81.26 & 93.72 & 97.19 & 97.83 & - \\
    \bottomrule
    \end{tabular}
    \vskip-5pt
\end{table}

\begin{figure}
    \centering
    \includegraphics[width=.7\linewidth]{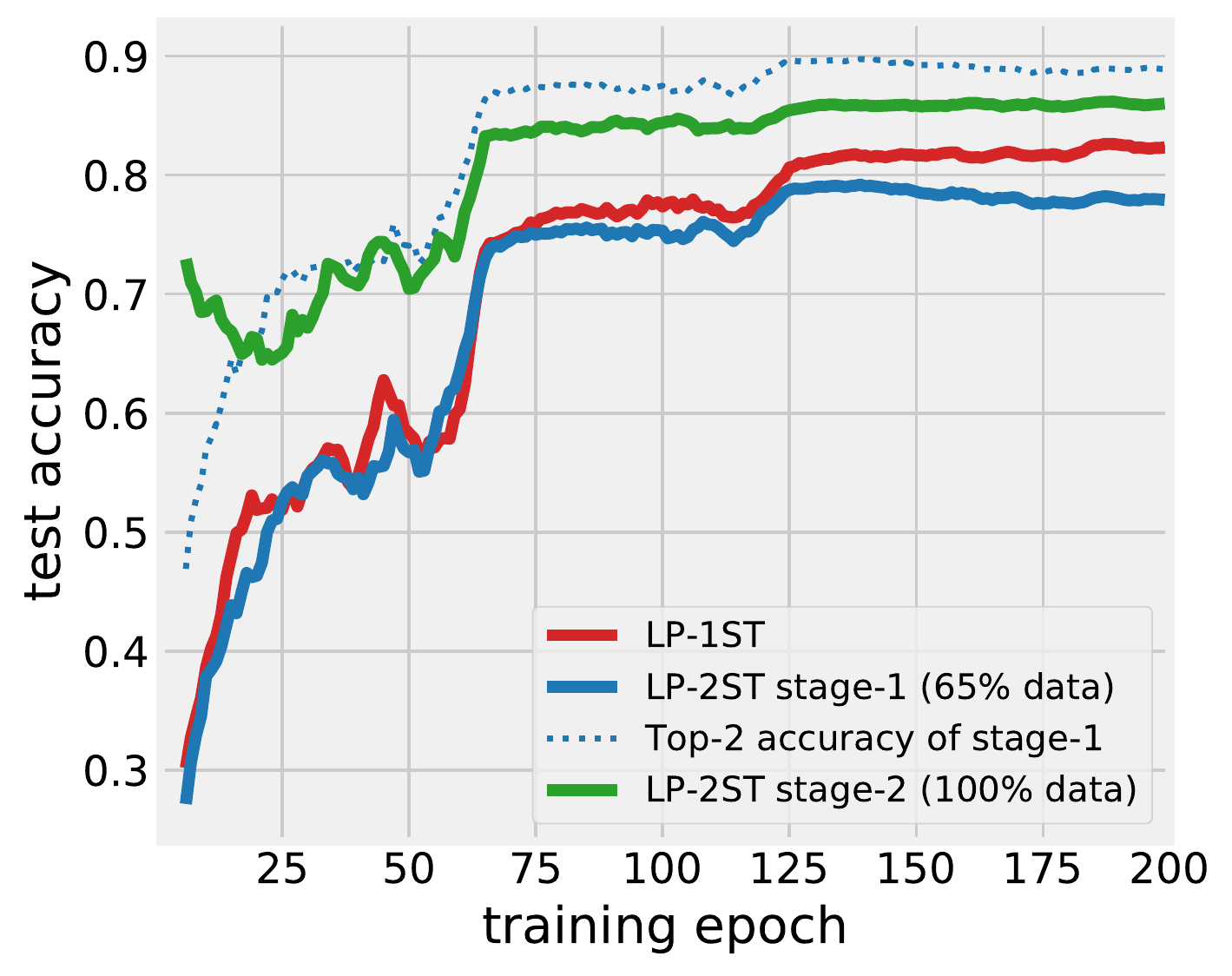}\vskip-8pt
    \caption{The learning curves of \oneST vs \twoST on CIFAR-10 ($\epsilon=2$).}
    \label{fig:cifar10-e2-stages}
\end{figure}

\section{Learning Dynamics of Multi-stage Training}
Fig.~\ref{fig:cifar10-e2-stages} visualizes the learning curves of \oneST and \twoST on CIFAR-10 with $\epsilon=2$.  Stage-1 of \twoST (using 65\% training data) clearly underperforms \oneST with the full training set. But it is good enough to provide useful prior for stage-2. The \RRP algorithm responds with an average $k=1.86$ over the remaining 35\% of the training set. As the dotted line shows, the top-2 accuracy of the model trained in stage-1 reaches 90\% at the end of training, indicating that the true label on the test set is within the top-2 prediction with high probability. In stage-2, we continue with the model trained in stage-1, and train on the combined data of the two stages. This is possible because the labels queried in stage-1 are already private. As a result, \twoST achieves higher performance than \oneST.

\section{Analysis of Robustness to Hyperparameters}

Following previous work,~\citep[e.g.,][] {papernot2020tempered}, we report the benchmark performance after hyperparameter tuning. In practice, to build a rigorous DP learning system, the hyperparameter tuning should be performed using private combinatorial optimization~\citep{gupta2010differentially}. Since that is not the main focus of this paper, we skip this step for simplicity. Meanwhile, we do the following analysis of model performance under variations of different hyperparameters, which shows that the algorithms are robust in a large range of hyperparameters, and also provides some intuition for choosing the right hyperparameters.

\begin{figure}
    \centering
    \begin{overpic}[width=.49\linewidth]{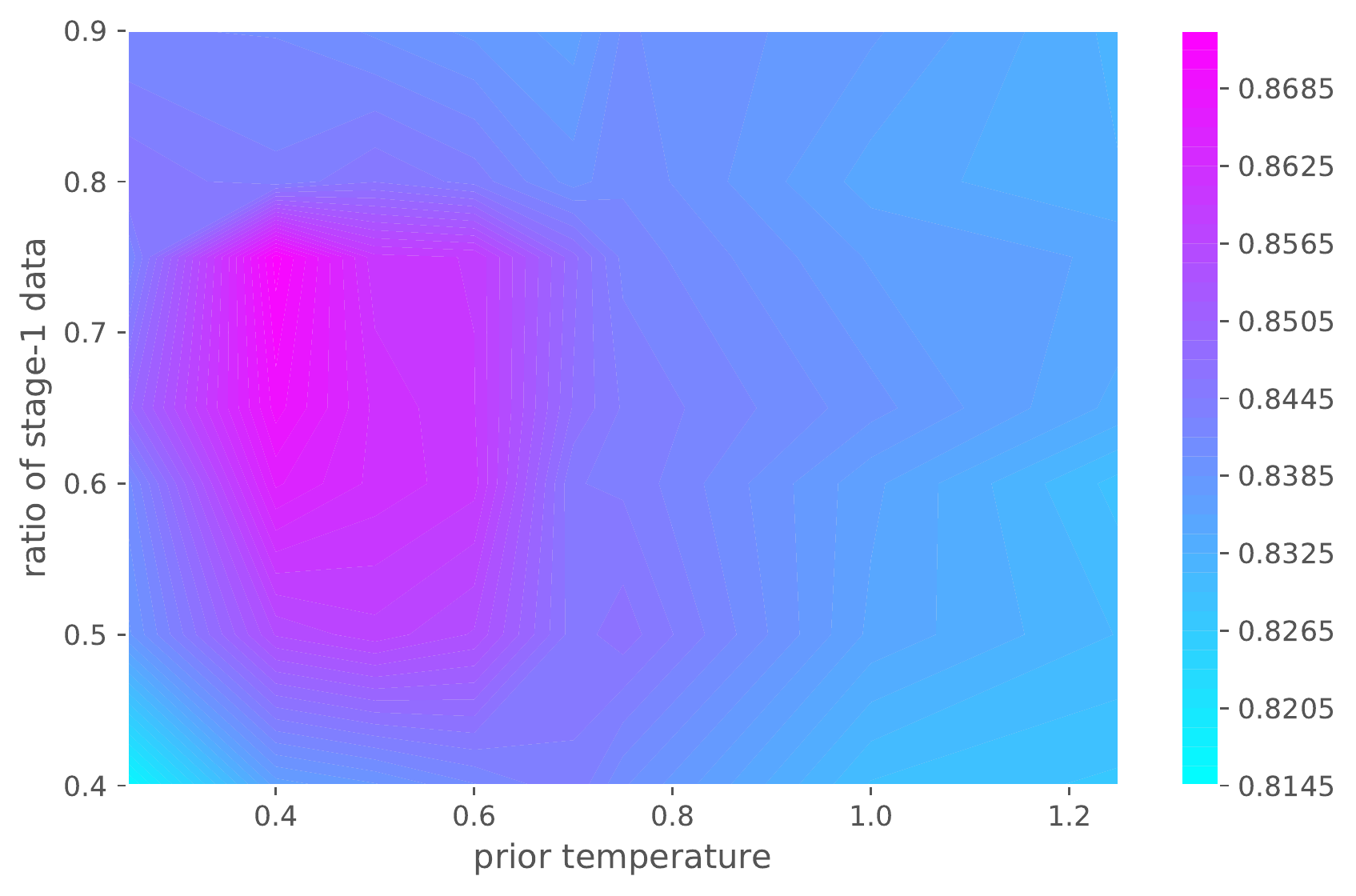}
    \put(0,0){\scriptsize (a)}\end{overpic}
    \begin{overpic}[width=.49\linewidth]{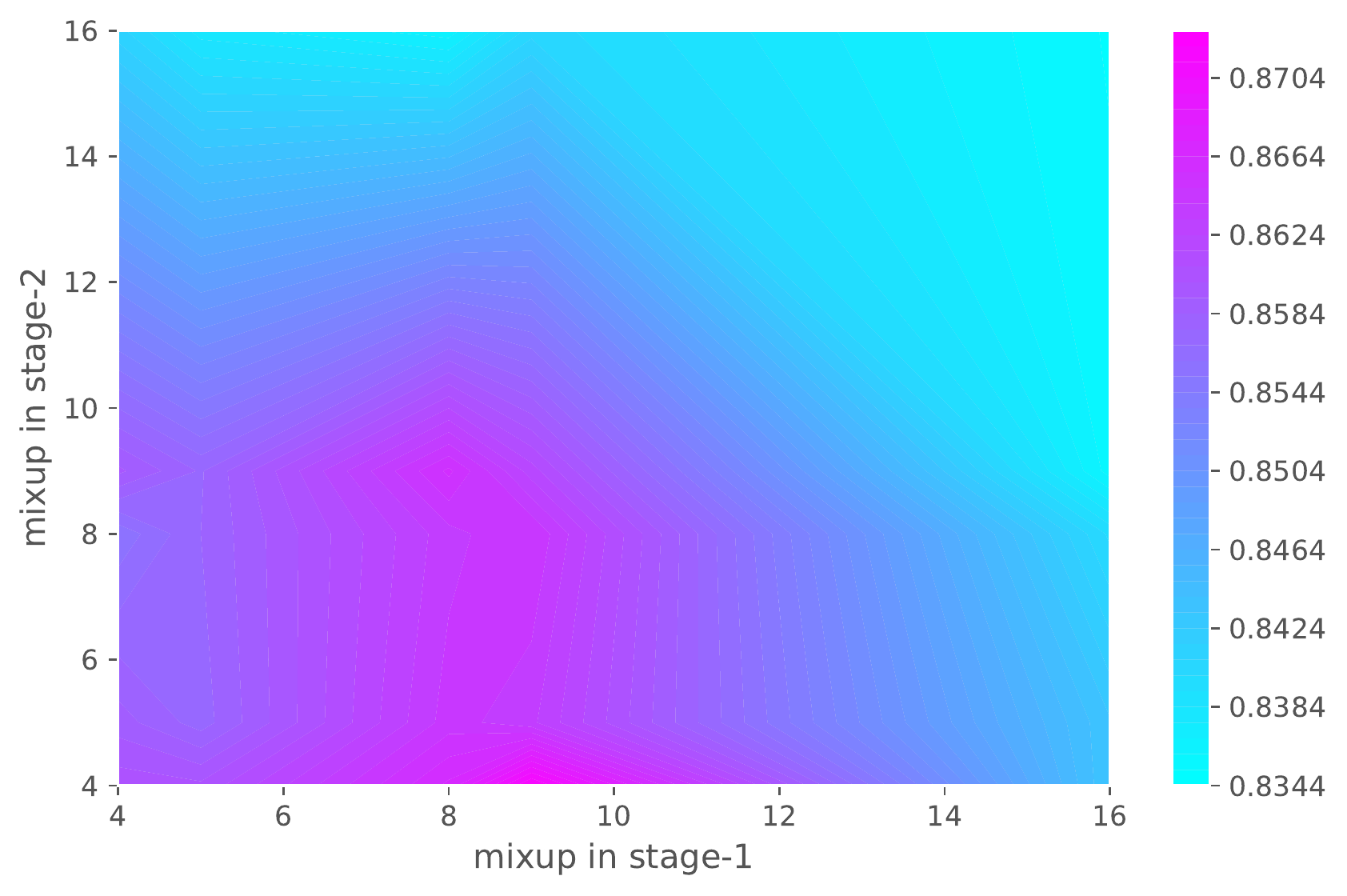}
    \put(0,0){\scriptsize (b)}\end{overpic}
            \vskip-8pt
    \caption{The final performance of \twoST on CIFAR-10 ($\epsilon=2$) (a) under different stage-1 / stage-2 data split and prior temperature; (b) under different mixup coefficients for stage-1 and stage-2.}
    \label{fig:temp-data-split}
\end{figure}

\para{Data Splits and Prior Temperature.}
The \emph{data split} parameter decides the ratio of data in different stages of training. Allocating more data for stage-1 allows us to learn a better prior model for the \twoST algorithm. However, it will also decrease the number of training samples in stage-2, which reduces the utility of the learned prior model. In practice,  ratios slightly higher than 50\% for stage-1 strike the right balance for \twoST.  We use a \emph{temperature} parameter $t$ to modify the learned prior. Specifically, let $f_k(x)$ be the logits prediction of the learned prior model for class $k$ on input $x$.  The temperature modifies the prior $\hat{p}_k(x)$ as:
\[
    \hat{p}^t_k(x) = \frac{\exp(f_k(x)/t)}{\sum_{k'=1}^K \exp(f_{k'}(x)/t)}.
\]

As $t\rightarrow 0$, it sparsifies the prior by forcing it to be more confident on the top classes, and as $t\rightarrow \infty$, the prior converges to a uniform distribution. In our experiments, we find it useful to sparsify the prior, and temperatures greater than $1$ are generally not helpful. Fig.~\ref{fig:temp-data-split}(a) shows the performance for different combinations of data split ratio and temperature.

\para{Accuracy of Stage-1.}
Ideally, one would want the $k$ calculated in \RRP to satisfy the condition that the ground-truth label is always in the top-$k$ prior predictions. Because otherwise, the randomized response is \emph{guaranteed} to be a wrong label. One way to achieve such a goal is to make the stage-1 model have high top-$k$ accuracy. For example, we could allocate more data to improve the performance of stage-1 training, or tune the temperature to spread the prior to effectively increase the $k$ calculated by \RRP. In either case, a trade-off needs to be made. In Fig.~\ref{fig:topk-violin}, we visualize the relation between top-$k$ test accuracy of stage-1 training and the final performance of \twoST. For each value range in the x-axis, we show the distribution of the final test accuracy where the average $k$ (rounded to the nearest integer) calculated in \RRP would make the top-$k$ accuracy of the corresponding stage-1 training fall into this value range. The plot shows that the final performance drops when the top-$k$ accuracy is too low or too high. In particular, achieving near perfect top-$k$ accuracy in stage-1 is \emph{not} desirable. Note this plot measures the top-$k$ accuracy on the \emph{test set}, so while it is useful to observe the existence of a trade-off, it does \emph{not} provide a procedure to choose the corresponding hyperparameters.

\begin{figure}
    \centering
    \includegraphics[width=.8\linewidth]{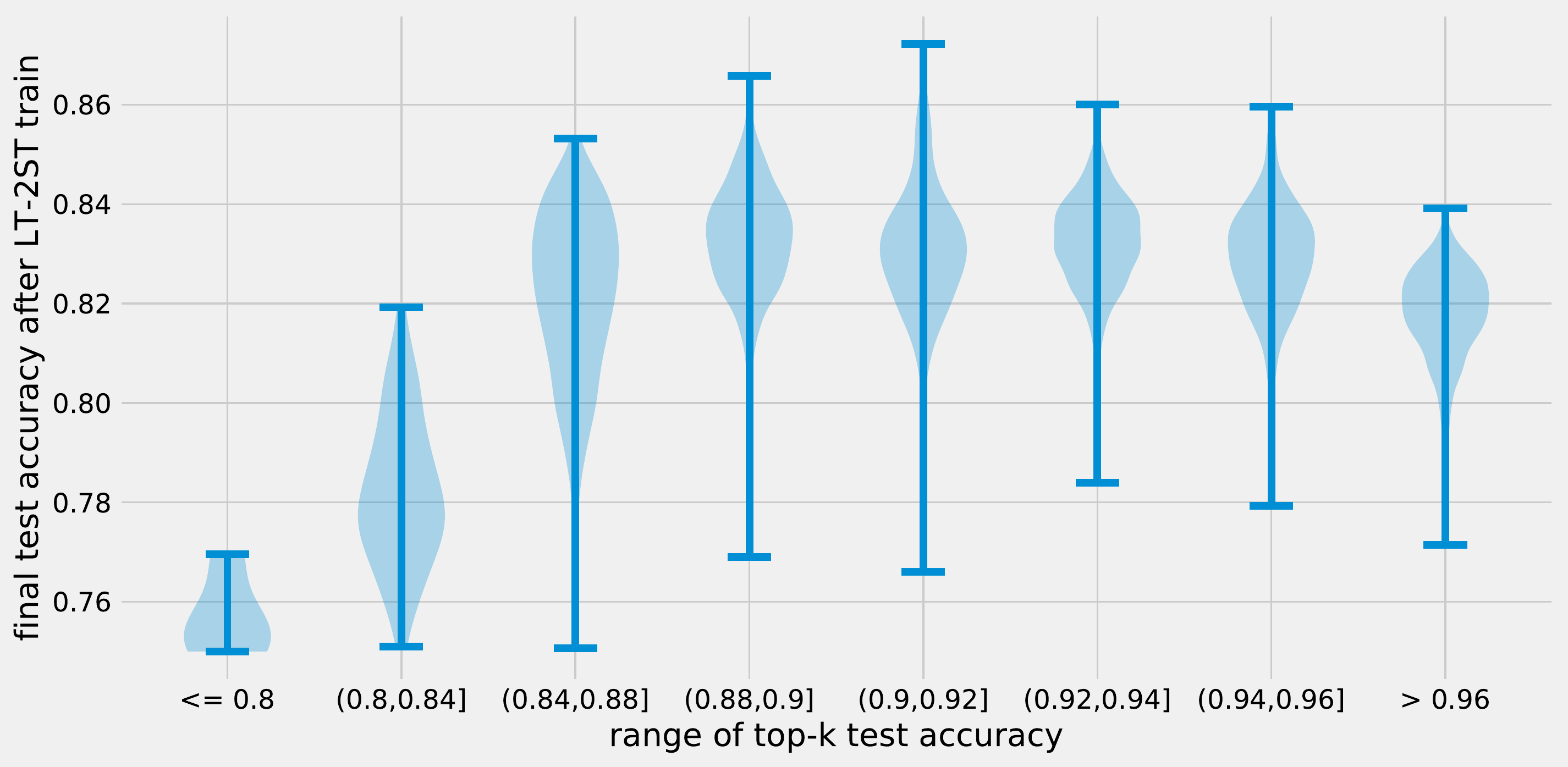}
    \caption{The relation between top-$k$ accuracy of stage-1 and the final accuracy of \twoST (CIFAR-10, $\epsilon=2$). The x-axis is the range of top-$k$ accuracy of stage-1 models evaluated on the test set. For each range, the violin plot
    shows the distribution of the final test accuracy of \twoST where the \RRP procedure calculated 
    an average $k$ (rounded to the nearest integer) for which the top-$k$ accuracy of the stage-1 model falls in the given range.}
    \label{fig:topk-violin}
\end{figure}

\para{Mixup Regularization.}
Mixup~\citep{zhang2017mixup} has a hyperparameter $\alpha$ that controls the strength of regularization (larger $\alpha$ corresponds to stronger regularization). We found that $\alpha$ values between 4 and 8 are generally good in our experiments, and as shown in Fig.~\ref{fig:temp-data-split}(b), stage-2 typically requires less regularization than stage-1. Intuitively, this is because the data in stage-2 is less noisier than stage-1.

\section{Convex SCO with \LDP}\label{sec:convex_erm}
In this section, we give the proofs of the Theorem \ref{prop:labelsgd-informal} and Corollary \ref{cor:topk-sco} for private stochastic convex optimization (SCO) and additionally prove some further, related results. We first formally introduce the setting of SCO.

Suppose we are given some \emph{feature space} $\MX$ (e.g., the space of all images), and \emph{label space} $[K] = \{ 1, 2, \ldots, K \}$. Write $\MZ = \MX \times [K]$. Let $\MW \subset \BR^p$ be a convex {\it parameter space}. Let $D$ be the (Euclidean) diameter of $\MW$, namely $D := \max_{w,w' \in \MW} \| w - w'\|$. Suppose we are given a loss function $\ell : \MW \times \MZ \ra \BR,
$ which specifies the loss $\ell(w, z)$ for a given parameter vector $w \in \MW$ on the example $z = (x,y)$. Given a sequence of samples $(x_1, y_1), \ldots, (x_n, y_n)$ drawn i.i.d.~from a distribution $P$ over $\MZ$, the goal is to find $w$ minimizing the popoulation risk, namely $\ML(w, P) := \E_{(x, y) \sim P}[\ell(w, (x, y))]$. Write $w^\st := \argmin_{w \in \MW} \ML(w,P)$. 
In this section, we make the following assumptions on $\ell$:
\begin{assumption}[Convexity]
  \label{asm:convexity}
  For each $z \in \MZ$, the function $w \mapsto \ell(w,z)$ is convex.
\end{assumption}

\begin{assumption}[Lipschitzness]
  \label{asm:lipschitz}
For each $z \in \MZ$, the function $w \mapsto \ell(w,z)$ is $L$-Lipschitz (with respect to the Euclidean norm). 
\end{assumption}

Under Assumptions \ref{asm:convexity} and \ref{asm:lipschitz}, \citet[Theorem 4.4]{BassilyFTT19} showed that there is an $(\ep, \delta)$-\DP algorithm that given $n$ i.i.d.~samples from a distribution $P$ and has access to a gradient oracle for $\ell$, outputs some $\hat w$ so that the excess risk is bounded as follows:
\begin{equation}
  \label{eq:log-factors}
    \E [\ML(\hat w, P)] - \ML(w^\st, P) \leq O \left( LD \cdot\left( \frac{\sqrt{p \log 1/\delta}}{n\ep}  + \frac{1}{\sqrt{n}} \right)\right).
\end{equation}
As shown by \citet{BassilyFTT19} (building off of previous work by \citet{bassily-private-2014}), the rate (\ref{eq:log-factors}) is tight up to logarithmic factors: in particular, there is a lower bound of $\Omega \left( \frac{\sqrt p}{n\ep} \right)$ on the excess risk for any $(\ep, \delta)$-DP algorithm, meaning that dimension dependence is necessary for private SCO. Subsequent work \citep{bassily_private_2020} showed how to obtain the rate (\ref{eq:log-factors}) in linear (in $n$) time. We additionally remark that there has much work (e.g., \citep{chaudhuri2011differentially,kifer_private_2012,bassily-private-2014,zhang_efficient_2017,wang_differentially_2017}) on the related problem of \emph{DP empirical risk minimization}, for which rates similar to (\ref{eq:log-factors}), except without the $1/\sqrt{n}$ term, are attainable. 

\subsection{Label-Private SGD}
In this section we prove Theorem \ref{prop:labelsgd-informal}, showing that \emph{dimension-independent} rates are possible in the setting of label DP privacy (in contrast to the standard setting of DP where privacy of the features must also be maintained). The algorithm that obtains the guarantee of Theorem \ref{prop:labelsgd-informal} is \LabelRRSGD (Algorithm \ref{alg:label-sgd-krr}).
Both \LabelRRSGD and the training procedure of Section \ref{sec:eval} (which uses \RRP) update the weight vectors using gradient vectors $\hat g_t$, which are obtained by using randomized response on the labels $y_t$ for the training examples $(x_t, y_t)$. \LabelRRSGD, however, ensures that $\hat g_t$ is an unbiased estimate of the true gradient, which facilitates the theoretical analysis, whereas this is not guaranteed the training procedure of Section \ref{sec:eval}.

\begin{algorithm}[!htp]
\small
  \caption{\bf \LabelRRSGD}\label{alg:label-sgd-krr}
  \textbf{Input:} Distribution $P$, convex and $L$-Lipschitz loss function $\ell$, privacy parameter $\ep$, convex parameter space $\MW$, variance factor $\sigma > 0$, step size sequence $\eta_t > 0$.
  \begin{enumerate}[leftmargin=14pt,rightmargin=20pt,itemsep=1pt,topsep=1.5pt]
  \item Choose an initial weight vector $w_1 \in \MW$.
  \item For $t = 1$ to $n$:
    \begin{enumerate}
    \item Receive a sample $(x_t, y_t) \sim P$.
    \item Let $\ty_t$ denote the output of $\RRShort(y_t)$. In other words,
    \begin{align*}
    \Pr[\ty_t = \hy] =
    \begin{cases}
    \frac{e^\ep}{e^\ep + K-1} & \text{ if } \hy = y_t \\
    \frac{1}{e^\ep + K-1} & \text{ if } \hy \ne y_t
    \end{cases}
    \end{align*}
    for all $\hy \in [K]$.
    \item Let $g_t = \grad_w \ell(w_t, (x_t, \ty_t))$ and
                                      \begin{align}
        \label{eq:def-hatgt}
\hat g_t = \frac{e^\ep+K-1}{e^\ep - 1} \cdot \left(g_t -  \sum_{k=1}^K \frac{\grad_w \ell(w_t, (x_t, k))}{e^\ep + K-1}\right).
\end{align}
\item Let $w_{t+1} \gets \Pi_{\MW} ( w_t - \eta_t \cdot \hat g_t)$. 
    \end{enumerate}
  \item Output $\hat w := w_{n+1}$. 
  \end{enumerate}
\end{algorithm}

We now restate Theorem \ref{prop:labelsgd-informal} formally below:
\begin{theorem}[Formal version of Theorem \ref{prop:labelsgd-informal}]
  \label{thm:labelsgd-formal}
  For any $\ep \in (0,1)$, the algorithm \LabelRRSGD satisfies the requirement of $\ep$-\LDP; moreover, if run with step size $\eta_t = \frac{D \ep}{6 KL \sqrt{t}}$, its output $\hat w$ satisfies
  $$
\E[\ML(\hat w, P)] - \ML(w^\st, P) \leq O \left( \frac{DL K \log(n)}{\ep \sqrt{n}}\right).
  $$
\end{theorem}
We remark that even in the non-private setting, a lower bound of $\Omega(DL/\sqrt{n})$ is known on the excess risk for stochastic convex optimization \cite{nemirovsky_problem_1983,agarwal_information_2009}, meaning that Theorem \ref{thm:labelsgd-formal} is tight up to a factor of $O(K \log n / \eps)$. In Section~\ref{sec:pop-lb}, we improve the lower bound to $\tilde{\Omega}(DL/\sqrt{\eps n})$ for small $\eps \leq 1$ (where $\tilde{\Omega}$ hides a logarithmic factor in $1/\eps$). Hence, our bound above is tight to within a factor of $\tilde{O}(K \log n / \sqrt{\eps})$.

\begin{proof}[Proof of Theorem \ref{thm:labelsgd-formal}]
We first verify the privacy property of \LabelRRSGD. For any two points $(x_t, y_t), (x_t, y_t')$, differing only in their label, if we let $\hat g_t, \hat g_t'$ be the vectors defined in (\ref{eq:def-hatgt}) for each of these points, respectively, then it is immediate from definition of $Q_t$ that for any subset $\MS \subset \BR^p$, $\frac{\Pr[\hat g_t \in \MS]}{\Pr[\hat g_t' \in \MS]} \leq e^\ep$. That \LabelRRSGD is $\ep$-\LDP follows immediately from the post-processing property of DP.
  
Next we establish the uility guarantee. Note that by definition of $\hat g_t$, we have that
\begin{align*}
  \E_{\ty_t} [\hat g_t] & = \frac{e^\ep + K-1}{e^\ep - 1} \cdot \left( \frac{e^\ep \cdot \grad_w \ell(w_t, (x_t, y_t))}{e^\ep + K-1} + \sum_{k\neq y_t} \frac{\grad_w \ell(w_t, (x_t, k))}{e^\ep + K-1} -  \sum_{k=1}^K \frac{\grad_w \ell(w_t, (x_t, k))}{e^\ep + K-1} \right) \\
  & = \grad_w \ell(w_t, (x_t, y_t)),
\end{align*}
i.e., $\hat g_t$ is an unbiased estimate of $\grad_w \ell(w_t, (x_t, y_t))$. 
  
  Next, we bound the variance of the gradient error $\hat g_t - \grad_w \ell(w_t, (x_t, y_t))$, as follows:
  \begin{align}
    & \E_{\ty_t} \left[ \left\| \hat g_t - \grad_w \ell(w_t, (x_t, y_t)) \right\|^2 \right] \nonumber\\
    \leq & 2\left(\frac{2K}{\ep} \right)^2 \cdot \E_{\ty_t} \left[ \left\| \left(g_t -  \sum_{k=1}^K \frac{\grad_w \ell(w_t, (x_t, k))}{e^\ep + K-1}\right)  \right\|^2 \right] + 2 \left\| \grad_w \ell(w_t, (x_t, y_t)) \right\|^2\nonumber\\
    \leq & \frac{32 K^2L^2}{\ep^2} + 2L^2 \leq \frac{36K^2 L^2}{\ep^2}\nonumber,
  \end{align}
  where we have used that $\ell$ is $L$-Lipschitz, $\ep \leq 1$, and that $K \geq 2$.

  Using \citet[Theorem 2]{shamir-stochastic-2013} with gradient moment $G^2 := \frac{36K^2 L^2}{\ep^2}$, we get that for step size choices $\eta_t := \frac{D}{G\sqrt{t}}$, the output $\hat w$ of \LabelRRSGD satisfies
\begin{align*}
& \E[\ML(\hat w, P)] - \ML(w^\st, P) 
\leq O \left( \frac{D G \log n}{\sqrt{n}} \right)  \leq O \left( \frac{DL K \log(n)}{\ep \sqrt{n}}\right). \qedhere
\end{align*}
\end{proof}

Now we prove Corollary \ref{cor:topk-sco}; a formal version of the corollary is stated below.
\begin{corollary}[Formal version of Corollary \ref{cor:topk-sco}] \label{cor:topk-sco-formal}
  Suppose that we are given a prior $\bp^x$ for every $x$ and let $Y^x_k$ denote the set of top-$k$ labels with respect to $\bp^x$. Then, for any $\ep \in (0,1)$, there is an $\ep$-\LDP algorithm which outputs $\hat w \in \MW$ satisfying
  \begin{align}
  \E[\ML(\hat w, P)] - \min_w \ML(w, P)\leq   {O}\left(DL \cdot \left(\frac{k \log n}{\ep \sqrt{n}} + \Pr_{(x, y) \sim P}[y \notin Y^x_k]\right)\right)\label{eq:topk-formal}
  \end{align}
\end{corollary}
\begin{proof}
  Suppose we are given access to samples $(x, y)$ drawn from a distribution $P$ on $\MX \times [K]$. For a pair $(x, y) \in \MX \times [K]$, define a random pair $\xi((x, y)) \in \MX \times [K]$, by setting $\xi((x, y)) = (x, y)$ if $y \in Y^x_k$, and otherwise letting $\xi((x, y))$ to be drawn uniformly over the set $\{ (x, k') : k' \in Y^x_k\}$. Let $P'$ be the distribution of $\xi((x, y))$, where $(x, y) \sim P$. For any $w_1, w_2 \in \MW$, it follows that
  \begin{align}
    & | (\ML(w_1, P) - \ML(w_2, P)) - (\ML(w_1, P') - \ML(w_2, P')) |\nonumber\\
    = & \left| \int_\MZ [\ell(w_1, (x, y)) - \ell(w_2, (x,y)) ]dP((x, y)) - \int_\MZ [ \ell (w_1, (x, y)) - \ell(w_2, (x,y))] dP'((x, y)) \right| \nonumber\\
    \leq & \left| \int_{\{ (x, y) : y \not \in Y^x_k \}} \left([\ell(w_1, (x, y)) - \ell(w_2, (x,y))] - [\ell(w_1, \xi((x, y))) - \ell(w_2, \xi((x,y)))] \right) dP((x, y)) \right|\nonumber\\
    \leq & 2DL \cdot \Pr_{(x,y) \sim P} [y \not \in Y_k^x],\label{eq:l-p-prime}
  \end{align}
  where the last step uses that $|\ell(w_1, (x,y)) - \ell(w_2, (x,y)) | \leq L \| w_1 - w_2 \| \leq LD$ for all $w_1, w_2 \in \MW$. 
  Now we simply run the algorithm \LabelRRSGD, except that when we receive a point $(x, y) \sim P$, we pass the example $\xi((x, y))$ to \LabelRRSGD (instead of $(x, y)$), and we let the set of possible labels be $Y_k^x$ (instead of $[K]$). Since each such example $\xi((x,y))$ is only passed to \LabelRRSGD once, the resulting allgorithm is still $\ep$-\LDP. Since the label of $\xi((x, y))$ belongs to $Y_k^x$, which has size $k$ for all $x$, Theorem \ref{thm:labelsgd-formal} gives that the output $\hat w$ of \LabelRRSGD satisfies $\E[\ML(\hat w, P')] - \min_{w} \ML(w, P') \leq O \left( \frac{DLk \log(n)}{\ep \sqrt{n}} \right)$. Next (\ref{eq:l-p-prime}) gives that, for any fixed $\hat w$, letting $w_{P'}^\st = \argmin_w \ML(w, P'), w_{P}^\st = \argmin_w \ML(w, P)$,
  \begin{align}
    \ML(\hat w, P) - \ML(w_{P}^\st, P) \leq & \ML(\hat w, P') - \ML(w_P^\st, P') + 2DL \cdot \Pr_{(x,y) \sim P} [y \not \in Y_k^x] \nonumber\\
    \leq &  \ML(\hat w, P') - \ML(w_{P'}^\st, P') + 2DL \cdot \Pr_{(x,y) \sim P} [y \not \in Y_k^x]\label{eq:switch-pprime},
  \end{align}
  where (\ref{eq:switch-pprime}) follows since $\ML(w_{P'}^\st, P') \leq \ML(w_P^\st, P')$ by definition of $w_{P'}^\st$. 
 (\ref{eq:topk-formal}) is an immediate consequence.
\end{proof}

\subsection{A Better Bound for Approximate DP}
Next we introduce an algorithm, \LabelNormalSGD\ (Algorithm \ref{alg:label-sgd}), which shows how to improve upon the excess risk bound of Theorem \ref{thm:labelsgd-formal} by a factor of $\sqrt{K}$, if we relax the privacy requirement to approximate \LDP (i.e., $(\ep, \delta)$-\LDP with $\delta > 0$). \LabelSGD performs a single pass of SGD over the input dataset, with the following modification: it adds a Gaussian noise vector to each gradient vector with nonzero variance {only in the $K$-dimensional subspace $\ML_t$ corresponding to the $K$ possible labels for each point $x_{t}$}. This means that the norm of a typical noise vector scales only as $\sqrt{K}$ as opposed to the scaling $\sqrt{p}$, which similar algorithms for the standard setting of DP (e.g., \cite{bassily-private-2014}) obtain.
\begin{algorithm}[!htp]
\small
  \caption{\bf \LabelNormalSGD}\label{alg:label-sgd}
  \textbf{Input:} Distribution $P$ over $\MX \times [K]$, convex and $L$-Lipschitz loss function $\ell$, privacy parameters $\ep, \delta$, convex parameter space $\MW$, variance factor $\sigma > 0$, step size sequence $\eta_t > 0$.
  \begin{enumerate}[leftmargin=14pt,rightmargin=20pt,itemsep=1pt,topsep=1.5pt]
  \item Choose an initial weight vector $w_1 \in \MW$.
  \item For $t = 1$ to $n$:
    \begin{enumerate}
    \item Receive $(x_t, y_t) \sim P$.
    \item Let $\tilde b_t \sim \MN(0, \sigma^2 I_p)$.
    \item Let $\ML_t \gets \Span \{ \grad_w \ell(w_t, (x_{t}, k)) : k \in [K] \} \subset \BR^p$.
    \item Let $b_t \gets \Pi_{\ML_t} (\tilde b_t)$ denote the Euclidean projection of $\tilde b_t$ onto $\ML_t$.
    \item Let $w_{t+1} \gets \Pi_{\MW} ( w_t - \eta_t \cdot (\grad_w \ell(w_t, (x_{t}, y_{t})) + b_t))$. 
    \end{enumerate}
  \item Output $\hat w := w_{n+1}$. 
  \end{enumerate}
\end{algorithm}

\begin{proposition}
  \label{prop:labelsgd}
  There is a constant $C > 0$ so that the following holds.
  For any $\ep, \delta \in (0,1)$, $\sigma = \frac{CL \sqrt{\log 1/\delta}}{\ep},\ \eta_t = \frac{D}{\sqrt{(L^2 + K\sigma^2) \cdot t}}$, the algorithm \LabelSGD (Algorithm \ref{alg:label-sgd}) is $(\ep, \delta)$-LabelDP and satisfies the following excess risk bound:
  $$
\E[\ML(\hat w, S)] - \ML(w^\st, S) \leq  O \left( \frac{DL \sqrt{K\log 1/\delta}\cdot \log(n)}{\ep \sqrt{n}}\right).
  $$
\end{proposition}

\iftrue
\begin{proof}[Proof of Proposition \ref{prop:labelsgd}]
  We first argue that the privacy guarantee holds. Note that for any $k,k' \in [n]$, for any $x \in \MX, w \in \MW$, we have $\| \grad_w \ell(w, (x,k)) - \grad_w \ell(w, (x,k')) \| \leq 2L$. Therefore, for any   $w_t \in \MW$, the mechanism
  $$
k \mapsto \grad_w \ell(w_t, (x_{i_t},k)) + b_t
$$
is $(\ep, \delta)$-\DP as long as $\sigma \geq \frac{CL \sqrt{\log 1/\delta}}{\ep}$, for some constant $C > 0$ \citep{dwork2014algorithmic}. Since each $(x_t, y_t)$ is used in only a single iteration of \LabelNormalSGD, it follows from the post-processing of DP that \LabelNormalSGD is $(\ep, \delta)$-\LDP for this choice of $\sigma$.

Next we establish the utility guarantee.
Since, for each $t \in [n]$, $\ML_t$ is a subspace of $\BR^p$ of at most $K$ dimensions, it holds that for each $t$,
$
\E [ \| b_t \|^2] \leq K\sigma^2.
$
Thus $\E \left[\|  \nabla_w \ell(w_t, (x_{i_t}, y_{i_t})) + b_t \|^2\right] \leq L^2 + K\sigma^2 $. 
Using \citet[Theorem 2]{shamir-stochastic-2013} with gradient moment $G^2 := L^2 + K\sigma^2$, we get that for step size choices $\eta_t := \frac{D}{G\sqrt{t}}$, it holds that
\begin{align*}
& \E[\ML(\hat w, S)] - \ML(w^\st, S) 
\leq O \left( \frac{D G \log n}{\sqrt{n}} \right) \leq O \left( \frac{DL \sqrt{K \log 1/\delta} \cdot \log(n)}{\ep \sqrt{n}}\right). \qedhere
\end{align*}
\end{proof}
\fi

\subsection{Lower Bound on Population Risk}
\label{sec:pop-lb}

In this section, we prove the following lower bound on excess risk, which is tight with respect to~\eqref{prop:labelsgd} in Proposition \ref{prop:labelsgd} up to a factor of $\tilde{O}(\sqrt{K/\eps})$.

\begin{proposition} \label{prop:lb-sco}
For any $\eps \in (0, 1], D, L > 0$ and any sufficiently large $n \in \N$ and sufficiently small $\delta > 0$ (both depending on $\eps$), the following holds: for any $(\eps, \delta)$-LabelDP algorithm $\AShort$, there exists a loss function $\ell$ that is $L$-Lipschitz and convex, and a distribution $P$ for which
\begin{align}
  \label{eq:label-priv-sco-lb}
\E_{\tS \sim P^{\otimes n}, \hw \sim \AShort(\tS)}[\ML(\hw, P)] - \ML(w^\st, P) \geq \tilde{\Omega}\left(\frac{DL}{\sqrt{\eps n}}\right).
\end{align}
\end{proposition}

We remark that the lower bound of $\Omega(DL/\sqrt{n})$ is well known for \emph{non-private} SCO. This lower bound applies to our setting as well and thus the lower bound in \Cref{prop:lb-sco} can be viewed as an improvement of a factor for $\tilde{\Omega}(1/\sqrt{\eps})$ over the non-private lower bound.

We prove \Cref{eq:label-priv-sco-lb} by first proving an analogous bound in the empirical loss minimization (ERM) setting and then deriving SCO via a known reduction.

\subsection{Lower Bound on Excess Risk for ERM}

Recall that in ERM setting, we are given a set $S = \{(x_1, y_1), \dots, (x_n, y_n)\} \subseteq \MZ$ of $n$ labelled examples. The empirical risk of $w$ is defined as $\ML(w, S) := \frac{1}{n} \sum_{i=1}^n \ell(w, (x, y))$. Here we would like to devise an algorithm that minimizes the excess empirical risk, i.e., $\E[\ML(\hw, S)] - \ML(w^\st, S)$ where $\hw$ is the output of the algorithm and $w^\st := \argmin_{w \in \MW} \ML(w, S)$.

We start by proving the following lower bound on excess risk for \LDP ERM algorithms. Note that the lower bound does not yet grow as $\eps$ decreases; that version of the lower bound will be proved later in this section.

\begin{proposition}
\label{prop:excessrisklb-erm}
For any $\eps, D, L, \delta > 0, K \geq 2$ and $n \in \N$ such that $\eps \leq O(1), \delta \leq 1 - \Omega(1)$, the following holds: for any $(\eps, \delta)$-LabelDP algorithm $\AShort$, there exists a loss function $\ell$ that is $L$-Lipschitz and convex, and a dataset $\tS$ of size $n$ for which
\begin{align}
  \label{eq:label-priv-lb}
\E_{\hw \sim \AShort(\tS)}[\ML(\hw, \tS)] - \ML(w^\st, \tS) \geq \Omega\left(\frac{DL}{\sqrt{n}}\right).
\end{align}
\end{proposition}

\iftrue
\begin{proof}
Let $\MW := \{w \in \R^d : \|\bw\| \leq D/2\}$ and $\MX := \{x \in \R^d: \|x\| \leq 1\}$. We define the loss to be
\begin{align*}
\ell(w, (x, y)) :=
\begin{cases}
L \cdot \left<w, x\right> & \text{ if } y = 1, \\
-L \cdot \left<w, x\right> & \text{ if } y = 2, \\
0 & \text{ otherwise.}
\end{cases}
\end{align*}
Note that the diameter of $\MW$ is $D$ and $\ell(\cdot, (x, y))$ is convex and $L$-Lipschitz. Consider any $(\eps, \delta)$-\LDP algorithm $\AShort$. Let $e_i \in \R^n$ be the $i$th standard basis vector. Consider a dataset $S = \{(e_1, y_1), \dots, (e_n, y_n)\}$ where $y_1, \dots, y_n \in \{1, 2\}$ are random labels which are $1$ w.p. $0.5$ and $2$ otherwise. For notational convenience, we write $\ty_i$ to denote $2y_i - 3 \in \{-1, 1\}$. By the $(\eps, \delta)$-\LDP guarantee of $\AShort$, we have
\begin{align*}
&\Pr_{S, \hw \sim \AShort(S)}[\ty_i \cdot \left<\hw, e_i\right> > 0] \\
&= \frac{1}{2} \Pr_{S, \hw \sim \AShort(S)}[\left<\hw, e_i\right> < 0 \mid \ty_i = -1] \\
&\qquad + \frac{1}{2} \Pr_{S, \hw \sim \AShort(S)}[\left<\hw, e_i\right> > 0 \mid \ty_i = 1] \\
&\leq \frac{1}{2} \cdot \left(e^{\eps} \cdot \Pr_{S, \hw \sim \AShort(S)}[\left<\hw, e_i\right> < 0 \mid \ty_i = 1] + \delta\right) \\
&\qquad + \frac{1}{2} \left(e^{\eps} \cdot \Pr_{S, \hw \sim \AShort(S)}[\left<\hw, e_i\right> > 0 \mid \ty_i = -1] + \delta\right) \\
&= e^{\eps} \cdot \Pr_{S, \hw \sim \AShort(S)}[\ty_i \cdot \left<\hw, e_i\right> < 0] + \delta.
\end{align*}
This implies that
\begin{align} \label{eq:dot-product-incorrect-sign}
\Pr_{S, \hw \sim \AShort(S)}[\ty_i \cdot \left<\hw, e_i\right> > 0] \leq \frac{e^{\eps} +  \delta}{e^{\eps} + 1}.
\end{align}
Letting $I_{\hw, S} := \{i \in [n] : \ty_i \cdot \left<\hw, e_i\right> > 0\}$ for any $S$, $\hw$,
\begin{align}
\E_{S, \hw \sim \AShort(S)}[|I_{\hw, S}|] &= \sum_{i \in [n]} \Pr_{S, \hw \sim \AShort(S)}[\ty_i \cdot \left<\hw, e_i\right> > 0] \nonumber \\ 
&\overset{\eqref{eq:dot-product-incorrect-sign}}{\leq} \left(\frac{e^{\eps} + \delta}{e^{\eps} + 1}\right)n. \label{eq:dot-product-incorrect-sign-size}
\end{align}
Consider any $S$ as generated above; it is obvious to see that $w^\st = \frac{D}{2} \cdot \left(\frac{1}{\sqrt{n}} \sum_{i \in [n]} \ty_i e_i\right)$, which results in $\ML(w^\st, S) = -\frac{DL}{2\sqrt{n}}$. On the other hand, for any $\hw$,
\begin{align}
& \ML(\hw, S)
= \frac{1}{n} \sum_{i \in [n]} \ell(\hw, (e_i, y_i)) = \frac{1}{n} \sum_{i \in [n]} -L \left<\hw, \ty_i \cdot e_i\right> \nonumber \\
&\geq \frac{1}{n} \sum_{i \in I_{\hw, S}} -L \left<\hw, \ty_i \cdot e_i\right> \nonumber = \frac{-L}{n} \left<\hw, \sum_{i \in I_{\hw, S}} \ty_i \cdot e_i\right> \nonumber \\
&\geq \frac{-L}{n} \cdot \|\hw\| \cdot \left\|\sum_{i \in I_{\hw, S}} \ty_i \cdot e_i\right\| \geq \frac{-L}{n} \cdot \frac{D}{2} \cdot \sqrt{|I_{\hw, S}|}, \label{eq:alg-loss}
\end{align}
where we used Cauchy--Schwarz inequality in the second inequality above. As a result, we have
\begin{align*}
&\E_S[\E_{\hw \sim \AShort(S)}[\ML(\hw, S)] - \ML(w^\st, S)] \\
&= \E_{S, \hw \sim \AShort(S)}[\ML(\hw, S)] + \frac{DL}{2\sqrt{n}} \\
&\overset{\eqref{eq:alg-loss}}{\geq} \frac{-DL}{2n} \cdot \E_{S, \hw \sim \AShort(S)}\left[\sqrt{|I_{\hw, S}|}\right] + \frac{DL}{2\sqrt{n}} \\
&\geq \frac{-DL}{2n} \cdot \sqrt{\E_{S, \hw \sim \AShort(S)}\left[|I_{\hw, S}|\right]} + \frac{DL}{2\sqrt{n}} \\
&\overset{\eqref{eq:dot-product-incorrect-sign-size}}{\geq} \frac{DL}{2\sqrt{n}} \left(-\sqrt{\frac{e^{\eps} + \delta}{e^{\eps} + 1}} + 1\right) \\
&\geq \Omega(DL/\sqrt{n}),
\end{align*}
where the second inequality follows from Cauchy--Schwarz inequality and the last inequality follows from our assumption that $\delta \leq 1 - \Omega(1)$ and $\eps \leq O(1)$.
\end{proof}
\fi

To make the lower bound above grows with $1/\sqrt{\eps}$ for $\eps \leq 1$, we will apply the technique used in \cite{SteinkeU16}. Recall that a pair of datasets are said to be $k$-neighbor if they differ in at most $k$ labels. The following is a well-known bound, so-called \emph{group privacy}; see e.g. \citet[Fact 2.3]{SteinkeU16}. (Typically this fact is stated for the standard DP but it applies to \LDP in the same manner.)

\begin{fact} \label{fact:group-dp}
Let $\AShort$ be any $(\eps, \delta)$-\LDP algorithm. Then, for any $k$-neighboring database $S, S'$ and every subset $T$ of the output, we have $\Pr[\AShort(S) \subseteq T] \leq e^{k\eps} \cdot \Pr[\AShort(S') \subseteq T] + \frac{e^{k\eps} - 1}{e^{\eps} - 1} \cdot \delta$.
\end{fact}

We can now prove the following lower bound that grows with $1 / \sqrt{\eps}$ by simplying replicating each element $1/\eps$ times.

\begin{lemma}
\label{lem:excessrisklb-erm-eps}
For any $\eps' \in (0, 1], D, L, \delta' > 0, K \geq 2$ and $n \in \N$ such that $n \geq 1/\gamma, \delta' \leq \Omega(\eps')$, the following holds: for any $(\eps', \delta')$-LabelDP algorithm $\AShort'$, there exists a loss function $\ell$ that is $L$-Lipschitz and convex, and a dataset $\tS'$ of size $n$ for which
\begin{align}
  \label{eq:label-priv-lb-eps}
\E_{\hw \sim \AShort'(\tS')}[\ML(\hw, \tS')] - \ML(w^\st, \tS') \geq \Omega\left(\frac{DL}{\sqrt{\eps' n}}\right).
\end{align}
\end{lemma}

\begin{proof}
Suppose for the sake of contradiction there exists $(\eps, \delta)$-\LDP algorithm $\AShort'$ such that $\E_{\hw \sim \AShort'(\tS')}[\ML(\hw, \tS')] - \ML(w^\st, \tS') \leq o\left(\frac{DL}{\sqrt{\eps' n}}\right)$. Let $k = \lfloor 1/\eps\rfloor$. We construct an algorithm $\AShort$ as follows: on input $\tS$, it replicates each element of $\tS$ $k$ times to construct a dataset $\tS'$. It then returns $\AShort'(\tS')$. From the utility guarantee of $\AShort'$, we have $\E_{\hw \sim \AShort(\tS)}[\ML(\hw, \tS)] - \ML(w^\st, \tS) \leq o\left(\frac{DL}{\sqrt{n}}\right)$. Furthermore, \Cref{fact:group-dp} ensures that $\AShort$ is $(\eps, \delta)$-DP for $\eps = k\eps' \leq 1$ and $\delta = \frac{e^{k \eps'} - 1}{e^{\eps'} - 1} \delta' \leq O(\delta' / \eps')$. When $\delta' = C/\eps'$ for any sufficiently small $C > 0$, $\AShort$ violates \Cref{prop:excessrisklb-erm}, concluding our proof.
\end{proof}

\subsection{From ERM to SCO}

Bassily et al.~\cite{BassilyFTT19}\footnote{See the proof in Appendix D of the arXiv version of their paper~\cite{bassily-sco-arxiv}.} gave a reduction from private SCO to private ERM. Although this bound is proved in the context of standard (both label and sample) DP, it is not hard to see that a similar bound holds for \LDP with exactly the same proof. To summarize, their proof yields the following bound:

\begin{lemma}
For any $\gamma, \eps > 0$ and $\delta \in (0, 1/2)$, suppose that there is an $(\frac{\eps}{4 \log(2/\delta)}, \frac{e^{-\eps} \delta}{8 \log(2/\delta)})$-\LDP algorithm that yields expected excess population risk of for SCO is at most $\gamma$. Then, there exists an $(\eps, \delta)$-\LDP algorithm for convex ERM (with the same parameters $D, L, n$) with excess empirical risk at most $\gamma$.
\end{lemma}

Plugging this into \Cref{lem:excessrisklb-erm-eps}, we arrive at \Cref{prop:lb-sco}.
\nc{\tP}{\tilde{P}}
\nc{\tell}{\tilde{\ell}}
\nc{\tML}{\tilde{\ML}}
\nc{\bP}{\bar{P}}
\nc{\bell}{\bar{\ell}}
\nc{\bML}{\bar{\ML}}

\section{Generalization Bounds for RR with Prior}

Let $\MX, \MZ$ be similar to the previous section and $\MY = [K]$ be the class of labels. We consider a setting where there is a concept class $\MF$ of functions $f: \MX \to \mathbb{R}$. Given $n$ samples drawn i.i.d. from some distribution $P$ on $\MZ$, we would like to output a function $f$ with a small \emph{population risk}, which is defined as $\ML(f; P) = \E_{(x,y) \sim P}[\ell(f(x), (x, y))].$, where $\ell: \mathbb{R} \times \MZ \to [0, 1]$ is a loss function. Throughout this section, we assume that $\ell$ is $L$-Lipschitz (\Cref{asm:lipschitz}).

\textbf{Priors and Randomized Response.} Let $k \le K$ be a positive integer. We work in the same setting as \Cref{cor:topk-sco-formal}, i.e. we assume a prior $\bp^x$ for every $x$ and let $Y^x_k$ denote the set of top-$k$ labels with respect to $\bp^x$. We let $\tP$ be the distribution where we first draw $(x, y) \sim P$ and then output $(x, \ty)$ where $\ty \sim \RRtopk_{\bp^x}(y)$ with DP parameter $\epsilon$.

\textbf{Debiased Loss Function.}
Let $p_{k, \eps}$ denote $\frac{1}{e^{\eps} + k - 1}$.
We consider a debiased version of the loss $\ell$; this was done before in \cite{natarajan2013learning} for the case of binary classification with noisy labels. In our setting, it generalizes to the following definition:
\begin{equation}\label{eq:ell_tilde_from_ell}
    \tilde{\ell}(t, (x, y)) := \frac{1}{1 - k \cdot p_{k, \epsilon}} \cdot \bigg(\ell(t, (x, y)) - \sum_{y' \in Y^{x}_k} p_{k, \epsilon} \cdot \ell(t, (x, y')) \bigg).
\end{equation}

For a set $S$ of $n$ labeled examples $(x_1, y_1), \dots, (x_n, y_n) \in \MZ$, its \emph{empirical risk} (w.r.t loss $\tell$) as $\tML(f; S) = \frac{1}{n} \sum_{i=1}^n \tell(f(x_i), (x_i, y_i)).$

We consider simple $\eps$-\LDP algorithm that randomly draws $n$ i.i.d. samples $S$ from $P$, apply ($\eps$-\LDP) $\RRtopk$ on each of the label to get a randomized dataset $\tS$, and finally apply empirical risk minimization w.r.t. the debiased loss function $\tell$ on $\tS$. We remark that this algorithm is exactly the same as drawing $n$ samples i.i.d. from $\tP$ and apply empirical risk minimization (again w.r.t. $\tell$). Our main result of this section is a generalization bound roughly saying that the empirical risk (w.r.t. $\tell$) is small iff the popultion risk (w.r.t. $\ell$) is small. This is stated more formally below, where $\MR_{n, D}(\MF)$ denote the Rademacher Complexity of $\MF$ (defined below in \Cref{def:rad}).

\begin{theorem}\label{th:RR_with_Prior_ERM}
Let $P_\MX$ be the marginal of $P$ over $\MX$. Let $\tS$ be a set of $n$ i.i.d. labeled samples drawn from $\tP$. Then, with probability at least $1-\beta$, the following holds for all $f \in \MF$:
\begin{equation}\label{ineq:fixed_k}
    |\tML(f; \tS) - \ML(f; P)| \le 2 L\cdot\frac{1 + k \cdot p_{k, \epsilon}}{1 - k \cdot p_{k, \epsilon}}\cdot \mathcal{R}_{n, P_\MX}(\mathcal{F}) + \sqrt{\frac{\log(2/\beta)}{2 n}} + \Pr_{(x,y) \sim P}[y \notin Y^x_k].
\end{equation}
\end{theorem}

Via standard techniques (see e.g. \cite{natarajan2013learning}), the above bound imply that the empirical risk minimizer incurs excess loss similar to the bound in~\Cref{ineq:fixed_k} (within a factor of 2).

Recall that $\RRtopk$ can of course be thought of $\RRP$ in the case when e.g. the prior $\bp^x$ is uniform over the $k$ labels in $Y^x_k$. Thus, \Cref{th:RR_with_Prior_ERM} can be viewed as a generalization bound for $\RRP$ with these ``uniform top-$k$'' priors. 
\subsection{Additional Preliminaries}

To prove~\Cref{th:RR_with_Prior_ERM}, we need several additional observations and definitions. In addition to the previously defined $\ML(f; P), \tML(f; S)$, we analogously use $\tML(f; P), \ML(f; S)$ to denote the population risk w.r.t. $\tell$ on distribution $P$ and the empirical risk w.r.t. $\ell$ on the labeled sample set $S$ respectively.

\paragraph{Properties of the Debiased Loss Function.}
We will start by proving a few basic properties of the debiased loss functions. The first two lemmas are simple to check:

\begin{lemma}\label{lem:debiasing}
If $y \in Y^x_k$, it holds that $\E_{\ty \sim \RRtopk_{\bp^x}(y)}[\tilde{\ell}(t, (x, \ty))] = \ell(t, (x, y))$.
\end{lemma}

\begin{lemma}\label{le:lip_tilde}
$\tilde{\ell}$ is $L \cdot \frac{1 + k  \cdot p_{k, \epsilon}}{1- k  \cdot p_{k, \epsilon}}$-Lipschitz (in $t$ for every fixed $x, y$).
\end{lemma}

Finally, we observe that the population risk w.r.t. $\tell$ on distribution $\tP$ is close to that w.r.t. $\ell$ on $P$:

\begin{lemma}\label{le:bias_bounding}
For any function $f$, we have
\begin{equation}\label{ineq:bias_bound_notin}
    |\ML(f; P) - \tML(f, \tP)| \le \Pr_{(x,y) \sim P}[y \notin Y^x_k].
\end{equation}
\end{lemma}

\begin{proof}
We can write
\begin{align*}
    |\ML(f; P) - \tML(f; \tP)| &= |\E_{(x, y) \sim P}[\ell(f, (x, y))] - \E_{(x, y) \sim P, \ty \sim \RRtopk_{\bp^x}(y)}[\ell(f, (x, \ty))]| \\
    &\le \E_{(x, y) \sim P}[|\ell(f, (x, y)) - \E_{\ty \sim \RRtopk_{\bp^x}(y)}[\ell(f, (x, \ty))]|].
\end{align*}
Due to \Cref{lem:debiasing}, the inner term is zero whenever $y \in Y^x_k$; furthermore, since the range of $\ell$ is in $[0, 1]$, the last term is at most $\Pr_{(x,y) \sim P}[y \notin Y^x_k]$ as desired.
\end{proof}

\paragraph{Rademacher Complexity.}
Given a space $\mathcal{V}$ and a distribution $D$ over $\mathcal{V}$, we let $S$ be a set of examples $v_1, \dots, v_n$ drawn i.i.d. from $D$. We also let $\mathcal{F}$ be a class of functions $f:\mathcal{V} \to \mathbb{R}$.
\begin{defn}[Empirical Rademacher Complexity]
The \emph{empirical Rademacher complexity} of $\mathcal{F}$ is defined as:
\begin{equation}
    \hat{\mathcal{R}}_{n, S}(\mathcal{F}) = \E_{\sigma_1, \dots, \sigma_n}\bigg[ \sup_{f \in \mathcal{F}} \bigg( \frac{1}{n} \sum_{i=1}^n \sigma_i f(v_i) \bigg)\bigg],
\end{equation}
where $\sigma_1, \dots, \sigma_n$ are i.i.d. random variables sampled uniformly at random from $\{\pm 1\}$.
\end{defn}

\begin{defn}[Rademacher Complexity] \label{def:rad}
The Rademacher complexity of $\mathcal{F}$ is defined as
\begin{equation}
    \mathcal{R}_{n, D}(\mathcal{F}) = \E[\hat{\mathcal{R}}_{n, S}(\mathcal{F})],
\end{equation}
where the expectation is over the randomness of the subset $S$ which consists of $n$ elements chosen i.i.d. from $D$.
\end{defn}

We also need the following two known lemmas.

\begin{lemma}[\cite{bousquet2003introduction}]\label{le:rademacher_generalization}
Let $D$ be a distribution and $\beta \in (0,1)$. If $\mathcal{F} \subseteq \{f: \mathcal{V} \to [0, 1]\}$ and $S = \{v_1, \dots, v_n\}$ consists of $n$ elements drawn i.i.d. from $D$, then with probability at least $1-\beta$ over the randomness of $S$, for every function $f \in \mathcal{F}$, it holds that
\begin{equation}
    \left|\E_{v \sim D}[f(v)] - \frac{1}{n} \sum_{i = 1}^n f(v_i)\right| \leq 2 \mathcal{R}_{n, D}(\mathcal{F}) + \sqrt{\frac{\ln(2/\beta)}{n}}.
\end{equation}
\end{lemma}

The following lemma is a standard bound for the empirical Rademacher complexity (and follows from the Ledoux-Talagrand contraction inequality \cite{ledoux2013probability}).
\begin{lemma}\label{le:lip_comp}
Let $\mathcal{F} \subseteq \{f: \mathcal{X} \to \mathbb{R}\}$. Let $S$ be a multiset of $n$ (possibly repeated) elements $v_1, \dots, v_n \in \mathcal{X}$. Moreover, let $\Phi_1, \dots, \Phi_n$ be $L$-Lipschitz functions mapping $\mathbb{R}$ to $\mathbb{R}$. Then, it holds that
\begin{equation}
    \E_{\sigma_1, \dots, \sigma_n}\bigg[ \sup_{f \in \mathcal{F}} \bigg( \frac{1}{n} \sum_{i=1}^n \sigma_i \Phi_i(f(v_i)) \bigg)\bigg] \le L \cdot \hat{\mathcal{R}}_{n, S}(\mathcal{F}).
\end{equation}
\end{lemma}

\subsection{Proof of~\Cref{th:RR_with_Prior_ERM}}

With the preliminaries ready, we can now prove~\Cref{th:RR_with_Prior_ERM}.

\begin{proof}[Proof of~\Cref{th:RR_with_Prior_ERM}]
With probability $1 - \beta$, the following holds:
\begin{align}
\sup_{f \in \MF} |\tML(f, S) - \ML(f, P)| 
&\leq \sup_{f \in \MF}\left(|\ML(f; P) - \tML(f, \tP)| + |\tML(f, S) - \tML(f, \tP)|\right) \nonumber \\
\text{(Lemma~\ref{le:bias_bounding})} &\le \Pr_{(x,y) \sim P}[y \notin Y^x_k] +  \sup_{f \in \MF}|\tML(f, S) - \tML(f, \tP)| \nonumber \\
&\leq \Pr_{(x,y) \sim P}[y \notin Y^x_k] + 2 \cdot \mathcal{R}_{n, D}(\tilde{\ell} \circ \mathcal{F}) + \sqrt{\frac{\ln(2/\beta)}{n}}\label{ineq:rademacher},
\end{align}
where inequality~(\ref{ineq:rademacher}) follows from Lemma~\ref{le:rademacher_generalization} with
\begin{equation*}
    \tilde{\ell} \circ \mathcal{F} := \bigg\{g: \mathcal{X} \times \mathcal{Y} \to [0,1], ~ g(x,y) = \tilde{\ell}(f(x), (x, y))| ~ f \in \mathcal{F} \bigg\}.
\end{equation*}

Finally, we have that:
\begin{align}
    \mathcal{R}_{n, D}(\tilde{\ell} \circ \mathcal{F}) &= \E_{S} \bigg[\E_{\sigma_1, \dots, \sigma_n}\bigg[ \sup_{f \in \mathcal{F}} \bigg( \frac{1}{n} \sum_{i=1}^n \sigma_i \tilde{\ell}(f(x_i), (x_i, y_i)) \bigg)\bigg]\bigg]\nonumber\\ 
    &\le \tilde{L} \cdot \E_{S} [\hat{\mathcal{R}}_{n, S_{\mathcal{X}}}(\mathcal{F})]\label{ineq:lipschitz_ub}\\ 
    &= \tilde{L} \cdot \mathcal{R}_{n, D_{\mathcal{X}}}( \mathcal{F})\label{ineq:tilde_L},
\end{align}
where~(\ref{ineq:lipschitz_ub}) follows from Lemma~\ref{le:lip_comp} (with $\Phi_i$ set to the function $\tilde{\ell}(\cdot, (x_i, y_i))$ for all $i \in \{1,\dots,n\}$, and with $S_{\mathcal{X}}$ denoting the projection of $S$ on $\mathcal{X}$), and from Lemma~\ref{le:lip_tilde} with
\begin{equation}\label{eq:L_tilde_def}
    \tilde{L} = L \cdot \frac{1 + k  \cdot p_{k, \epsilon}}{1- k  \cdot p_{k, \epsilon}}.
\end{equation}
Inequality~(\ref{ineq:fixed_k}) now follows by combining~(\ref{ineq:rademacher}),~(\ref{ineq:tilde_L}), and~(\ref{eq:L_tilde_def}).
\end{proof}

\end{document}